\documentclass{article}

\PassOptionsToPackage{numbers, compress}{natbib}
\usepackage[preprint]{neurips_2026}

\usepackage[utf8]{inputenc} % allow utf-8 input
\usepackage[T1]{fontenc}    % use 8-bit T1 fonts
\usepackage{hyperref}       % hyperlinks
\usepackage{url}            % simple URL typesetting
\usepackage{booktabs}       % professional-quality tables
\usepackage{amsfonts}       % blackboard math symbols
\usepackage{nicefrac}       % compact symbols for 1/2, etc.
\usepackage{microtype}      % microtypography
\usepackage{xcolor}         % colors
\usepackage{graphicx}
\usepackage{caption}  % For customizing captions
\usepackage{subcaption} % For creating subfigures
\usepackage{amsthm}
\usepackage{amssymb}
\usepackage{amsmath}
\usepackage{wrapfig}
\usepackage{CJKutf8}
\usepackage{fontawesome5}
\usepackage{amssymb}

\newcommand{\zh}[1]{\begin{CJK}{UTF8}{gbsn}#1\end{CJK}}
\theoremstyle{plain}
\newtheorem{theorem}{Theorem}[section]
\newtheorem{lemma}[theorem]{Lemma}

\newtheorem{assumption}[theorem]{Assumption}
\newtheorem{corollary}[theorem]{Corollary}

\theoremstyle{definition}
\newtheorem{definition}[theorem]{Definition}

\theoremstyle{remark}

\usepackage{enumitem}

\usepackage[compact]{titlesec}

\usepackage{listings}
\usepackage[most]{tcolorbox}
\tcbuselibrary{breakable,skins}

\definecolor{PromptBlue}{HTML}{0B6FAE}

\newtcblisting{promptbox}[1]{%
  enhanced,
  breakable,
  colback=white,
  colframe=PromptBlue,
  boxrule=0.8pt,
  arc=3mm,
  left=6pt,right=6pt,top=6pt,bottom=6pt,
  title={#1},
  colbacktitle=LightCyan,
  coltitle=black,
  fonttitle=\bfseries,
  boxed title style={boxrule=0pt,arc=3mm,left=8pt,right=8pt,top=4pt,bottom=4pt},
  listing only,
  listing options={%
    basicstyle=\ttfamily\footnotesize,    
    columns=fullflexible,
    keepspaces=true,
    breaklines=true,
    breakautoindent=false, % <-- important
    breakindent=0pt,       % <-- important
    showstringspaces=false,
    escapeinside={(*@}{@*)}
  }%
}

\definecolor{LightCyan}{RGB}{238,246,255}
\definecolor{WhiteColor}{RGB}{255,255,255}

\titlespacing{\section}{0pt}{0.9ex}{0.9ex}
\titlespacing{\subsection}{0pt}{0.9ex}{0.9ex}
\titlespacing{\subsubsection}{0pt}{0.9px}{0.9px}
\titlespacing*{\paragraph}{0pt}{0.5ex}{1ex}
\expandafter\def\expandafter\normalsize\expandafter{%
    \normalsize%
    \setlength\abovedisplayskip{5pt}%
    \setlength\belowdisplayskip{5pt}%
    \setlength\abovedisplayshortskip{0pt}%
    \setlength\belowdisplayshortskip{0pt}%
}

\definecolor{validgreen}{HTML}{009E73}
\definecolor{invalidred}{HTML}{D55E00}
\title{Sampling More, Getting Less: Calibration is the Diversity Bottleneck in LLMs}

\hypersetup{
    colorlinks=true,
    urlcolor=blue,
    linkcolor=blue,
    citecolor=blue
}

\newcommand{\icmlspade}{\ensuremath{\textcolor[HTML]{990000}{\spadesuit}}}
\newcommand{\icmlclub}{\ensuremath{\textcolor[HTML]{004977}{\clubsuit}}}

\author{
\parbox{\textwidth}{\centering
\textbf{Amin Banayeeanzade\thanks{Equal Contribution.}\,\,\,$^{\icmlspade}$} \quad
    \textbf{Qingchuan Yang\footnotemark[1]\,\,\,$^{\icmlspade}$} \quad
    \textbf{Dhruv Tarsadiya$^{\icmlspade}$} \quad
    \textbf{Fatemeh Bahrani$^{\icmlspade}$} \quad \vspace{0.4em} \\
    \textbf{Leonardo Blas$^{\icmlspade}$}\quad
    \textbf{Alfy Samuel$^{\icmlclub}$} \quad 
    \textbf{Robin Jia$^{\icmlspade}$} \quad \vspace{0.4em} \\
    \textbf{Meisam Razaviyayn$^{\icmlspade}$} \quad
    \textbf{Sai Praneeth Karimireddy$^{\icmlspade}$} \quad \vspace{0.6em} \\
    \textnormal{$^{\icmlspade}$University of Southern California \quad \quad $^{\icmlclub}$Capital One} \quad
    \vspace{0.6em} \\
    \texttt{\{banayeea,qcyang,razaviya,karimire\}@usc.edu} \vspace{0.6em}\\
    \includegraphics[height=1em]{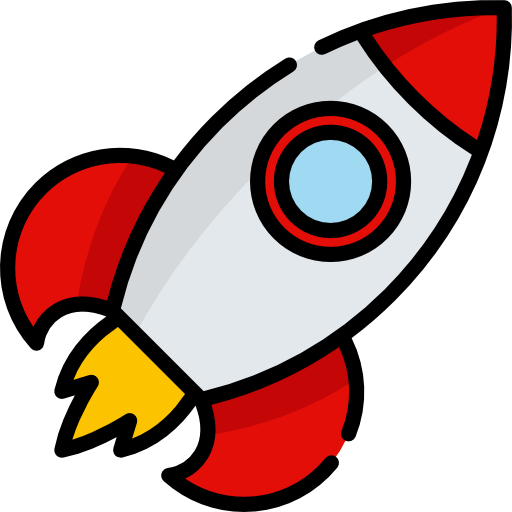} \, Demo: \texttt{\href{https://diversitycalibration.github.io}{https://diversitycalibration.github.io}} \quad\quad \vspace{-3mm}
}
}

\begin{document}

\maketitle

\begin{abstract}
Diversity is essential for language-model applications ranging from creative generation to scientific discovery, yet modern LLMs often collapse into a narrow subset of plausible outputs. While prior work has developed benchmarks for measuring this lack of diversity, less is known about how the step-by-step probability distributions at inference time cause the problem. We introduce a validity--diversity framework that attributes diversity collapse to how an LLM allocates probability mass across valid and invalid continuations during decoding. This framework decomposes the bottleneck into two complementary forms of miscalibration. First, \emph{order calibration}: valid tokens are not reliably ranked above invalid tokens, so rank-based cutoff rules must trade off between recovering valid continuations and admitting invalid ones. Second, \emph{shape calibration}: probability mass is overly concentrated only on few valid continuations while having a heavy-tail of mixed valid and invalid tokens, so maintaining high validity limits diversity. We formalize both mechanisms and show that local failures compound across decoding steps, producing strong sequence-level losses in diversity. Empirically, we develop controlled diagnostics for probing these bottlenecks, including tasks with exactly known valid sets and oracle cutoff baselines. Across 14 language models spanning multiple families and scales, we find that diversity collapse is not merely a limitation of particular sampling heuristics, but a consequence of order and shape miscalibration in the LLM distribution.
\end{abstract}

\section{Introduction}

\begin{figure}[ht]
    \centering
      \begin{subfigure}[!t]{0.780\textwidth}
        \includegraphics[width=1.0\linewidth]{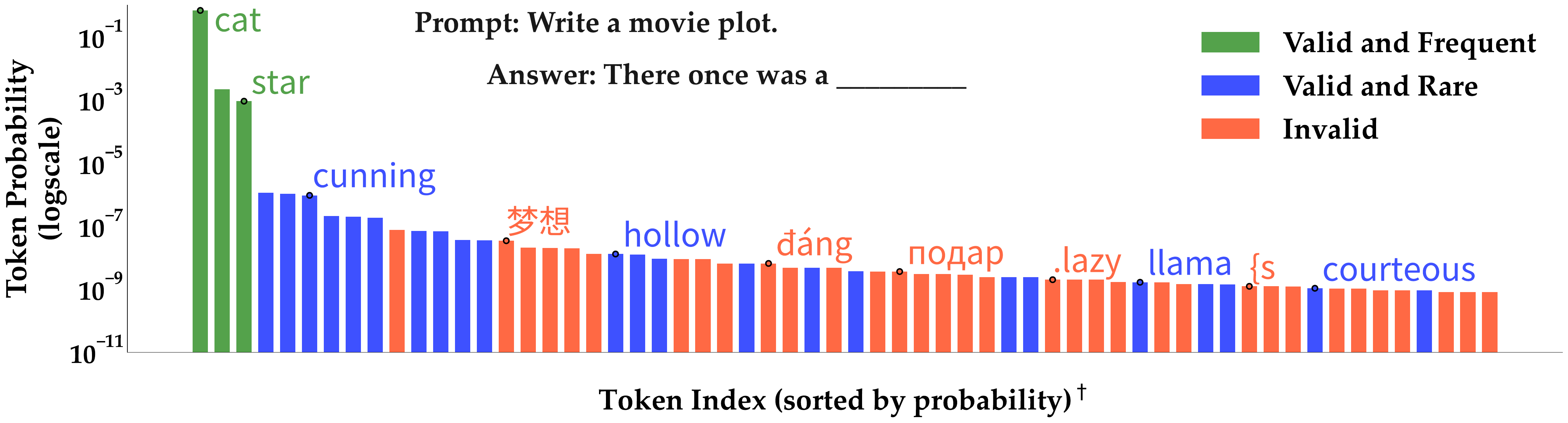}
      \end{subfigure}
       \hfill
      \begin{subfigure}[!t]{0.210\textwidth}
        \includegraphics[width=\linewidth]{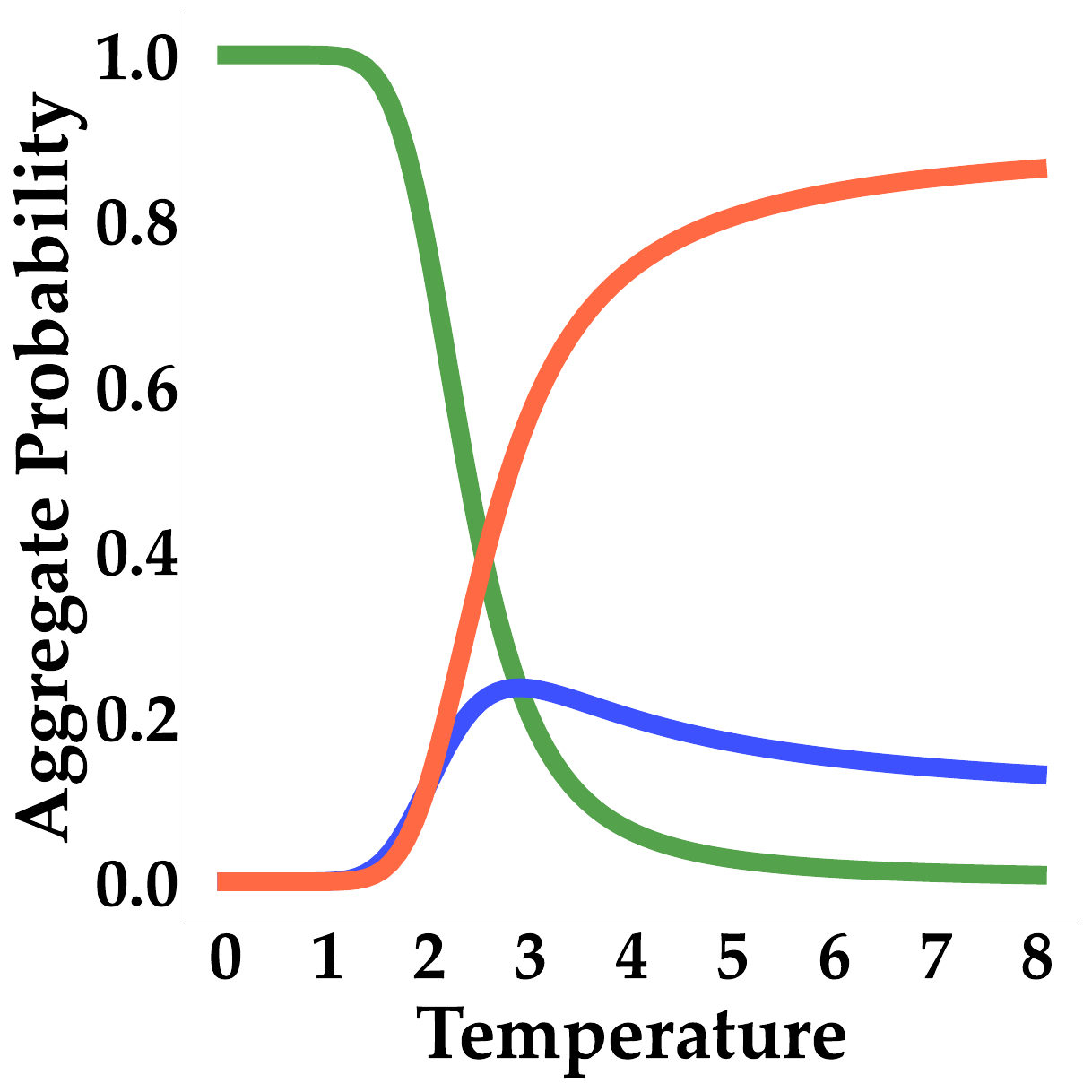}
      \end{subfigure}
    \caption{\textbf{(Left)} The token distribution of a generation step from \textsc{Qwen3.5-35B-A3B}. The distribution is very sharp in the front, followed by a heavy tail with mixed valid and invalid tokens. As a result, \textbf{(Right)} many valid tokens are unlikely to appear in the output under any temperature sampling. {$^\dagger$ tokens} are subsampled non-uniformly for enhanced visualization.}
    \label{fig:logits}
\end{figure}

Diversity in generation is essential for a wide range of applications, including synthetic data generation~\citep{banayeeanzade2026epsvec}, creative writing~\citep{zhang2025noveltybench}, recommendation systems~\citep{recommendation_diversity}, coding~\citep{yue2025limit-of-rlvr}, and exploration for scientific discovery~\citep{tu2026shared}. However, contemporary large language models (LLMs) often exhibit a notable lack of diversity~\citep{jiang2025artificial, murthy-etal-2025-one, sourati2025shrinking, Echoes2025Xu}. For instance, $59\%$ of stories in a \texttt{GPT-4}-generated dataset begin with \textit{``Once upon a time''}, and \texttt{GPT-5.5} repeatedly outputs \textit{``Valparaíso, Chile''} when asked to name a random city in the world (see Appendix~\ref{app:gpt}). 

These examples illustrate a broader failure mode: when generation is overly concentrated on a small set of high-probability outputs, many valid alternatives may be systematically undersampled. Recent
work has made substantial progress in measuring this lack of diversity~\citep{zhang2025noveltybench}, but they neither explain the source of the collapse nor provide diagnostic tools to systematically trace the problem.

Although standard sampling methods intend to resolve the problem, they rather reveal the difficulty; temperature sampling flattens the distribution, but often shifts probability mass toward invalid or nonsensical continuations before sufficient diversity is recovered~\citep{shypula2025evaluating}. Top-token filtering methods such as top-$k$ and min-$p$~\citep{minh2025turning} truncate the ranked distribution, but they either drop many valid alternatives or include invalid tokens. This suggests that the bottleneck is not merely the decoding heuristic, but the properties of the LLM distribution itself. We therefore ask: 
\begin{center}
\vspace{-0.5mm}
\textit{What are the distributional properties of LLMs that constrain their
ability \\ to generate outputs that are both valid and diverse?}
\vspace{-0.5mm}
\end{center}
To formalize this view, we study the inference-time distribution of an LLM through the lens of \emph{validity--diversity trade-off}. Rather than only measuring diversity on completed generations, we analyze the decoding process itself and how the model's next-token distribution allocates probability mass across generations. This perspective reveals two distinct failure modes.

First, LLMs fail in \emph{order calibration}: valid tokens are not reliably ranked above invalid ones. In Figure~\ref{fig:logits}, we show that many valid alternative tokens (blue) appear farther down the ranked distribution and are interleaved with invalid ones (red). When this occurs, any top-token filtering rule faces an unavoidable trade-off: expanding the cutoff recovers more valid continuations but also admits more invalid ones, while tightening it preserves validity but excludes valid alternatives. 

Second, LLMs fail in \emph{shape calibration}: probability mass is non-uniformly concentrated on a small number of valid continuations (green), but much smaller probability is assigned to other valid tokens (blue), while having a heavy tail of many invalid tokens (red). In the right panel of Figure~\ref{fig:logits}, we show that increasing temperature shifts probability mass away from the head, but much of this mass flows into the invalid tail rather than recovering rare valid alternatives. We formalize these effects and show that the resulting validity--diversity loss compounds when generating longer sequences.

Across $14$ language models spanning multiple families and scales, we find that these calibration failures constitute the primary bottleneck to diversity and addressing them unlocks broader diversity in model outputs. Our findings also have implications for model training and design, suggesting directions for mitigating these bottlenecks at their source. Finally, our analysis challenges implicit assumptions underlying common sampling strategies and provides guidance for more principled benchmarking and evaluation of decoding methods.

\paragraph{Contributions.}
Building on this framework, our contributions are:
\begin{enumerate}[left=2ex, topsep=1pt, itemsep=1pt, parsep=1pt]
    \item We introduce a framework for analyzing the validity--diversity trade-off at both token and sequence levels.
    \item We introduce \emph{order calibration} and \emph{shape calibration} as two complementary distributional bottlenecks. We theoretically and empirically demonstrate that local failures compound over sequence length.
    \item We develop controlled empirical diagnostics for probing these bottlenecks, including settings with exactly known valid sets and oracle cutoff baselines, and show that no decoding method that relies on top-token filtering can effectively recover diversity.
\end{enumerate}

\section{Related Work}

\paragraph{Diversity and mode collapse.}

Limited output diversity has emerged as a central failure mode of modern language models~\citep{Echoes2025Xu, ghafouri2026varianceparadoxaireduces, guo-etal-2025-benchmarking-linguistic, yu2025generationspacesizeunderstanding}. Homogeneity of the generation appears both within a single model across repeated samples~\citep{sourati2025shrinking} and across different models on the same prompts~\citep{jiang2025artificial}. Even frontier models remain substantially less diverse than humans~\citep{zhang2025noveltybench}, especially the models with extensive alignment and post-training~\citep{yang2025llm, west2025basemodelsbeataligned, kirk2024understandingeffectsrlhfllm, karan2025reasoningsamplingbasemodel, murthy-etal-2025-one}. Together, these works motivate studying diversity as a first-class property of language generation.

\paragraph{Evaluation.}
Diversity is inherently multi-dimensional, and recent work has moved beyond narrow lexical metrics toward broader assessments of open-ended generation \citep{zhang2025noveltybench,jiang2025artificial, karanjai2025evaluatingqualityrandomnessentropy}. Attempts to improve diversity might lead to text degeneration~\citep{Holtzman2020The}, and temperature should be carefully tuned~\citep{troshin2025controltemperatureselectivesampling, zhou2025balancingdiversityriskllm}. Therefore, diversity should not be assessed in isolation from quality~\citep{schaeffer2025minpmaxexaggerationcritical}, and raw diversity is misleading when many outputs are low-quality~\citep{shypula2025evaluating, wang2025optimizing}. Following this perspective, we treat diversity as meaningful only insofar as it broadens the space of valid, useful outputs, and we use this lens to diagnose when existing decoding rules fail to do so.

\paragraph{Improving diversity.}
A long line of work has sought to improve diversity through prompting~\citep{zhang2025verbalized, misaki2026stringseedthoughtprompting, wang2025multilingualpromptingimprovingllm}, training~\citep{li2025jointlyreinforcingdiversityquality, ismayilzada2025creativepreferenceoptimization, li2025preservingdiversitysupervisedfinetuning, chung2025modifyinglargelanguagemodel}, base-aligned model collaboration~\citep{wang2025optimizing, peeperkorn2025mindgapconformativedecoding}, and inference-time interventions~\citep{beam, su2023contrastive}. Closer to our work, inference-time stochastic methods such as temperature scaling \citep{ACKLEY1985147}, top-$k$ \citep{fan2018hierarchical}, top-$p$ \citep{Holtzman2020The}, and min-$p$ \citep{minh2025turning} sampling modify the support or sharpness of the next-token distribution.
% , with nucleus sampling in particular highlighting how high-probability decoding can itself induce degeneration \citep{Holtzman2020The}. 
Subsequent methods make truncation more adaptive \citep{truncation, meisterlocally, minh2025turning, potraghloo2025tophdecodingadaptingcreativity, zhu2024improvingopenendedtextgeneration}. Our work complements this literature: rather than proposing another decoding strategy, we examine why existing sampling rules often fail to recover meaningful diversity.

\section{Preliminaries}

We consider an auto-regressive LLM with vocabulary $\mathcal{V}$. Given a prompt $x\in \mathcal{V}^*$ and a generated prefix $y_{<t} = (y_1,\dots,y_{t-1}) \in \mathcal{V}^{t-1}$, the model defines a conditional distribution
$p(\cdot \mid y_{<t}, x)$
over the next token. The probability of a complete output $y = (y_1,\dots,y_d)$ of length $d$ is
$p(y \mid x) = {\prod_{t=1}^d p(y_t \mid y_{<t}, x)}$. When the task is fixed and unambiguous, we omit $x$ from our notations. 

We use $V \subseteq \mathcal{V}^*$ to represent the set of all valid responses to $x$. With slight misuse of notation, we say $y_{<t} \in V$, if there exists a continuation $w \subseteq\mathcal{V}^*$ such that the concatenation $y_{<t} \circ w \in V$.

\begin{definition}[Validity and Diversity]
\label{def:dv}
For a prompt $x$, let $Y \sim p(\cdot \mid x)$ denote the model's distribution over complete responses.

\begin{enumerate}[left=2ex, topsep=1pt, itemsep=1pt, parsep=1pt]
\item \textbf{(Validity)} We define validity as the total probability mass that the model assigns to $V$:
\[
\mathrm{Val}(p) := p(Y \in V \mid x) = \sum_{y \in V} p(y \mid x).
\]

\item \textbf{(Diversity)} We assume that all valid responses are equally preferred, and hence define diversity as the normalized effective support size of the model distribution restricted to the valid set:
\[
\mathrm{Div}(p)
:= \frac{e^{H(\tilde{p})}}{|V|},
\]
where $\tilde{p}=p(Y\mid Y\in V, x)$, and $H$ is the Shannon entropy.
\end{enumerate}

Intuitively, validity captures the probability mass that the model assigns to valid responses, and diversity then quantifies the effective coverage of valid outputs under the distribution restricted to the valid set~\citep{hillNumber, entropyanddiversity}. In a related work, \citet{yang2025llm} defines the exponential of entropy to be a token-level measure of the effective number of plausible next steps during generation. Both validity and diversity take values in the interval $[0,1]$, with higher values indicating better performance. Moreover, we define: %The goal of our paper is to identify the bottlenecks in decoding that make such trade-offs.

\end{definition}

\begin{definition}[Valid Continuations and Valid Tokens]
\label{def:valid_tokens}
Given a context $y_{<t} \in \mathcal{V}^{t-1}$, valid continuations are defined as the number of all sequences in $V$ that begin with the prefix $y_{<t}$,
\[
N(y_{<t}) := \bigl|\{z \in \mathcal{V}^* : y_{<t}\circ z \in V \}\bigr|,
\]

and the set of valid tokens is accordingly defined as tokens that lead to at least one valid continuation,
\[
G(y_{<t}) := \{v \in \mathcal{V} : N(y_{<t} \circ v)>0 \}.
\]

%These are precisely the next tokens that preserve access to at least one valid continuation.
\end{definition}

A decoding rule achieves high validity if it assigns high probability only on valid tokens at each time step, and it achieves high diversity if it explores many distinct tokens in $G$ rather than concentrating on only a few of them. However, LLMs, regardless of the decoding strategy employed, often exhibit a pronounced validity--diversity trade-off. In this work, we identify two primary sources of this phenomenon arising from properties of the model distribution: In \S\ref{sec:weak_calibration}, we first introduce \emph{order calibration} and its implications on top-token filtering methods. Next in \S\ref{sec:strong_calibration}, we identify the \emph{shape calibration} issue and we show that together, these effects constitute the primary sources of the observed validity--diversity trade-off.

\section{Order Calibration Fails: Valid tokens are not ranked first}\label{sec:weak_calibration}
Modern decoding strategies implicitly assume that valid tokens are concentrated near the top of the ranked distribution and that simple statistics of the distribution (e.g., cumulative mass or relative probability gaps) can reliably identify and retain these tokens. Under this view, diversity is increased by expanding the retained set, while validity is preserved by truncating low-probability regions.

In this section, we show that the LLM token distributions systematically violate these assumptions. Valid tokens are not confined to the head but are frequently interspersed with invalid tokens throughout the tail (see Figure~\ref{fig:logits}), and the relationship between rank and validity is neither monotone nor stable across contexts. As a result, any decoding rule based solely on rank-based filtering faces an inherent limitation: it must inevitably trade off between excluding valid tokens and admitting invalid ones. Even small imperfections in separating valid and invalid tokens at each step compound multiplicatively over long generations, leading to a sharp degradation in reachable valid outputs.

\paragraph{Cutoff strategies.}
We abstract all the top-token filtering methods as cutoff strategies. Let $S$ denote a cutoff strategy. Given a prefix $y_{<t}$, it first sorts the tokens by their conditional probabilities, and then selects a cutoff index, retaining all tokens up to that index and discarding the rest. Importantly, our framing strictly contains any adaptive top-token filtering method, since $S$ is not a predetermined rule; we allow it to be any arbitrary cutoff strategy, potentially depending on the prefix $y_{<t}$ and the model distribution $p(\cdot \mid y_{<t})$ at each step. Let $S_t(y_{<t})$ denote the set of the retained tokens. An ideal strategy $S$ would include as many valid tokens as possible (high recall) while excluding all invalid tokens (high precision). Therefore, we define the following to measure the quality of a strategy:

\begin{definition}[Precision/Recall]
\label{def:pr}
Let $V \subseteq \mathcal{V}^d$ denote the set of valid sequences,
$G(y_{<t})$ the set of valid next tokens, $N(y_{<t})$ the number of valid continuations, and $S_t(y_{<t})$ the set of retained tokens for a prefix $y_{<t}$.

\begin{enumerate}[left=2ex, topsep=1pt, itemsep=1pt, parsep=1pt]
    \item \textbf{(Local Precision)} 
    We define local precision as the fraction of retained tokens that are valid,
    \[
    \mathrm{Prec}_t(S;y_{<t})
        :=
    \frac{|S_t(y_{<t}) \cap G(y_{<t})|}{|S_t(y_{<t})|}.
    \]
    
    \item \textbf{(Local Recall)} We define local recall as the fraction of valid continuations that remain reachable after truncation.
    \[
        \mathrm{Rec}_t(S;y_{<t})
        :=
        \frac{
        \sum_{v \in S_t(y_{<t}) \cap G(y_{<t})}
        N(y_{<t}\circ v)
        }{
        N(y_{<t})
        }.
    \]
\end{enumerate}
Moreover, let $Q_S$ denote the sequence distribution induced by uniformly sampling from the retained sets. Then,
\begin{enumerate}[start=3, left=2ex, topsep=1pt, itemsep=1pt, parsep=1pt]
\item \textbf{(Sequence-Level Precision)} We define sequence precision as the probability of generating a valid sequence, i.e., \[\mathrm{Prec}_{\mathrm{seq}}(S):=Q_S(Y \in V).\]
\item \textbf{(Sequence-Level Recall)} We define sequence recall as the fraction of valid sequences that remain reachable.
\[\mathrm{Rec}_{\mathrm{seq}}(S):=\frac{
|\{y \in V : y_t \in S_t(y_{<t}) \ \forall t\}|
}{|V|}\]
\end{enumerate}
\end{definition}

\paragraph{Relation to validity and diversity.}
Sequence precision coincides exactly with validity. Sequence recall captures a complementary notion of diversity: it measures how much of the valid output
space remains accessible under the decoding rule. While our definition of diversity is entropy-based, recall provides a notion of coverage. In particular, if recall is small, then the decoder can only explore a small subset of valid outputs, regardless of how probability is distributed within it.

\subsection{A controlled testbed for order calibration}
\label{sec:order_testbed}
Our goal is to empirically measure the precision--recall trade-off introduced by any cutoff strategy.
A central challenge in practice is that the valid set $V$ is intractable to characterize, as it grows exponentially with sequence length, especially in open-ended tasks such as storytelling. To address this, we propose a practical procedure to approximate the precision/recall metrics. Detailed setup, models, and prompts are found in Appendix~\ref{app:pr-setup}.
\begin{figure}[t]
    \centering
      \begin{subfigure}[!t]{0.64\textwidth}
        \includegraphics[width=1.0\linewidth]{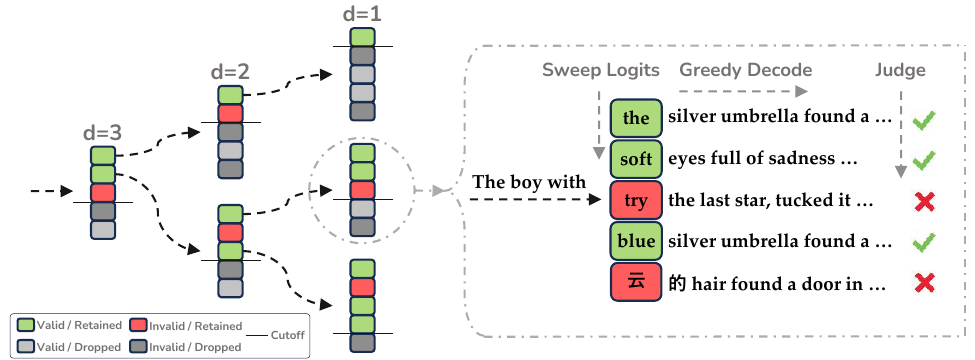}
      \end{subfigure}
       \hfill
    {\color{gray}\vrule width 0.2pt}
    \hfill
      \begin{subfigure}[!t]{0.30\textwidth}
      \includegraphics[width=\linewidth]{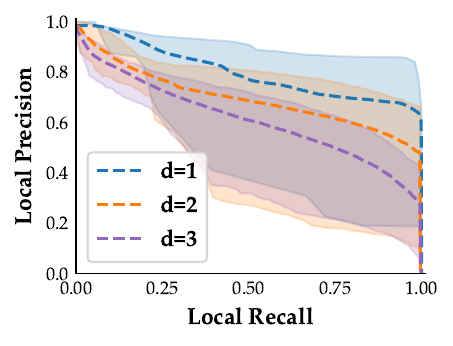}
      \end{subfigure}
    \caption{\textbf{(Left)} We sweep the logits and cutoff thresholds at each conditional distribution to enumerate retained tokens up to a certain depth, followed by greedy decoding from each leaf. A judge model then evaluates the validity of the generated sequences, allowing us to attribute a validity label to each token. We then measure the number of valid/invalid tokens that were retained/dropped by the cutoff strategy to obtain local precision and recall. \textbf{(Right)} The frontier precision--recall trade-offs at different depths obtained by sweeping cutoff strategies. The trade-off degrades as depth increases.}
    \label{fig:labeling_and_pr}
\end{figure}
\paragraph{Measuring token validity.} Given a fixed prefix $y_{<t}$, we query the LLM to get a sorted list of all next token candidates $\{v_1,\cdots,v_\mathcal{|V|}\}$. For each token $v_j$, we construct the extended prefix $y_{<t} \circ v_j$ and then perform greedy decoding to completion. This approximates the model’s most likely continuation conditioned on having selected $v_j$ (see Figure~\ref{fig:labeling_and_pr}, middle panel). We then evaluate the resulting sequence using an LLM-as-a-judge~\cite{gu2024survey}, scoring grammar, semantic, and overall validity. By thresholding this score, we obtain a binary validity label for the token $v_j$. Repeating this process for all tokens yields an approximate assessment of token-level validity given the prefix, and allows us to compute precision--recall trade-off as a function of the cutoff strategy in the next sections.

Appendix~\ref{app:llm-as-a-judge-prompts} provides additional details, including the evaluation rubric. Appendix~\ref{app:llm-as-a-judge-results} validates the reliability of the LLM judge against human annotations and examines the impact of judge model choice. Moreover, Appendix~\ref{app:greedy_vs_sampling} compares greedy completions vs sampling and shows that the results are robust to this choice.

\begin{wrapfigure}{r}{0.32\textwidth}
\vspace{-2mm} 
\includegraphics[width=\linewidth]{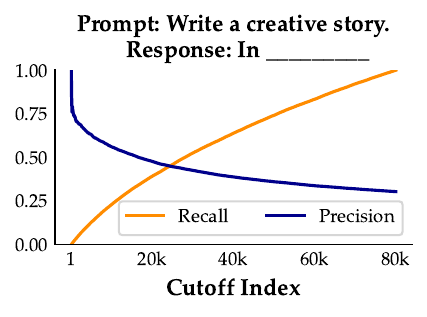}
\vspace{-6mm}    \caption{Local precision--recall trade-off when sweeping the cutoff in a single generation step.}
    \label{fig:1d_sweep}
\end{wrapfigure}

\paragraph{Single-step precision--recall trade-off.} The extracted labels allow us to measure the precision--recall trade-off at a single decoding step to observe how valid tokens are distributed in the conditional. In Figure~\ref{fig:1d_sweep}, we sweep the cutoff from the first index up to the token at rank $80$k on a story generation task with \texttt{Qwen3.5-35B}. Plotting the precision and recall at each cutoff, we observe that precision drops sharply at the front, but recall slowly improves, even at high token indexes. This indicates that there is a strong precision--recall trade-off and order calibration is severely violated on this single conditional. We provide qualitative examples in Appendix~\ref{app:pr-examples}.

\textbf{Multi-step precision--recall trade-off.} We extend our single-step methodology to a more realistic setup that captures the multi-step effects. Particularly, we extend the procedure to depth ${d > 1}$: Instead of immediately decoding greedily, we recursively expand tokens at each successive step, constructing a tree of continuations (Figure~\ref{fig:labeling_and_pr}, left). After expanding all nodes up to $d$ tokens at each branch, we complete each leaf via greedy decoding. By evaluating these resulting sequences with a judge, we obtain a validity label for each token in this sequence along the tree: a token is valid if at least one sequence containing it is valid. Since the number of sequences grows exponentially with depth and each depth requires many LLM calls, we sweep up to depth $3$ and subsample tokens at each depth, as detailed in Appendix~\ref{app:pr}.

Given the validity label of every token up to step $d$, we construct the precision--recall trade-off curves by sweeping over all possible cutoff strategies, and computing the local trade-off for every node using Definition~\ref{def:pr}. Note that our framework includes all top-token filtering strategies, as it allows any node to adjust its own cutoff arbitrarily. Each cutoff strategy gives a single point in the precision--recall curve; we take the Pareto frontier of all strategies as a representative of the best achievable trade-off.

\paragraph{Local precision--recall trade-off worsens with depth.}

We perform the above procedure across $10$ seeds, each repeated with a random query from NoveltyBench~\citep{zhang2025noveltybench} and a random prefix $y_{<t}$. For each seed, we compute the Pareto-optimal precision--recall curve at every node in the generation tree as previously mentioned. Figure~\ref{fig:labeling_and_pr} (right) summarizes the results. To analyze the effect of horizon length, we group nodes by their depth in the generation tree and report, for each depth, the maximum, minimum, and average Pareto frontiers. 
We observe the following:
\begin{itemize}[left=2ex, topsep=1pt, itemsep=1pt, parsep=1pt]
    \item Even the optimal cutoff strategy exhibits a non-negligible local precision--recall trade-off.
    \item The trade-off worsens as depth increases (from $d=1$ to $d=2$, and from $d=2$ to $d=3$).
    \item This degradation is not only present in the average, but also on the maximum frontier.
\end{itemize}

Overall, these results show that the precision--recall trade-off induced by cutoffs worsens with decoding depth. This provides empirical evidence that order calibration is frequently violated: valid continuations are not reliably ranked above invalid ones, and this misalignment compounds over longer horizons.

\subsection{Local order failures compound into sequence-level collapse}
\label{sec:order_compounding}
These measurements are at the level of local precision and recall. At the sequence level, the trade-off is stronger: even small but constant local imperfections lead to a dramatic collapse in global diversity since the errors incurred at each depth compound multiplicatively, as we show in the following:
\begin{theorem}[Compounding effect of decoding steps]
\label{thm:mult}
Suppose that at least \(m\) decoding positions exhibit a constant local precision--recall trade-off: at each such position, high local precision must discard a constant fraction of valid
continuations. Then there exist constants \(c,C>0\) such that any cutoff strategy
\(S\) satisfying
\[
\mathrm{Prec}_{\mathrm{seq}}(S) \ge 1-\delta
\]
must satisfy
\[
\mathrm{Rec}_{\mathrm{seq}}(S)
\le
\left(1-\delta\right)^{-C}
e^{-cm}.
\]
\end{theorem}

\begin{wrapfigure}{r}{0.43\textwidth}
\centering
\begin{minipage}{0.41\textwidth}
\vspace{-3mm}
\includegraphics[width=\linewidth]{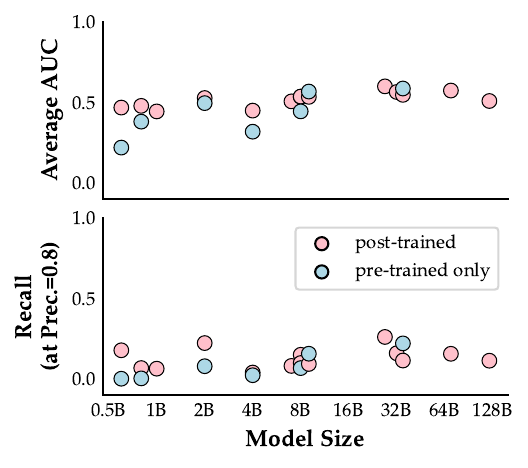}
\vspace{-6mm}
    \caption{
    Precision--recall trade-offs across \texttt{Qwen-3}, \texttt{Llama-3}, \texttt{Olmo-3} on 9 sizes and training stage. 
    Evaluations are averaged on 3 random positions and queries.
    \textbf{(Top)} Average area under the precision--recall frontier.
    \textbf{(Bottom)} Average recall at precision \(0.8\).
    % Across both metrics, larger values indicate better order calibration and a weaker local precision--recall trade-off. 
    }
    \label{fig:model_size}

  \vspace{0.5em}

\centering
\captionof{table}{Semantic and Lexical Diversity of Cutoff Strategies. Higher Embedding Diversity score corresponds to higher semantic diversity; lower Self-BLEU score corresponds to higher lexical diversity.}
\label{tab:results}
\renewcommand{\arraystretch}{0.7}
\setlength{\tabcolsep}{2.4pt}
\resizebox{1.0\textwidth}{!}{
\begin{tabular}{lcc}
\toprule
Strategy & Emb. Diversity ($\uparrow$) & Self-BLEU ($\downarrow$) \\
\midrule
oracle & $\bf{0.40_{\pm 0.15}}$ & $\bf{0.69_{\pm 0.21}}$ \\
top-$k$ & $0.33_{\pm 0.14}$ & $0.71_{\pm 0.17}$ \\
min-$p$ & $0.29_{\pm 0.13}$ & $0.80_{\pm 0.16}$ \\
top-$p$ & $0.25_{\pm 0.12}$ & $0.86_{\pm 0.13}$ \\
no filtering & $0.25_{\pm 0.11}$ & $0.86_{\pm 0.12}$ \\
\bottomrule
\end{tabular}}
\vspace{-5mm}
\end{minipage}
\vspace{-1mm}
\end{wrapfigure}
\paragraph{Interpretation.} The theorem formalizes a compounding effect. High sequence-level validity leaves only a small total budget for local precision errors. Hence, at most \(O(\log(1/(1-\delta)))\) decoding steps can tolerate such trade-offs; the remaining steps must use near-perfect cutoffs. If valid and invalid tokens are interleaved, such cutoffs necessarily discard a constant fraction of valid continuations. These losses compound multiplicatively, so the reachable valid set shrinks exponentially. In this sense, any cutoff rule can maintain validity only by sacrificing broad validity, thereby proving a hardness result for any decoding strategy that relies on top-token filtering. The details of the proof are in Appendix~\ref{app:mult}.
% \begin{wraptable}{r}{0.4\textwidth}
% \end{wraptable}
\subsection{Scaling and the diversity gap}
Sections~\ref{sec:order_testbed} and~\ref{sec:order_compounding} establish local precision--recall as a meaningful diagnostic for order calibration. We now use this diagnostic to quantify the practical importance of order calibration.
\paragraph{Model Size.} 
Figure~\ref{fig:model_size} summarizes the local precision--recall trade-offs across model families, scales, and training stages (see Appendix~\ref{app:pr-model-size}). The average AUC of the precision--recall frontier exhibits a mild upward trend with model size. However, the improvement is modest and far from eliminating the trade-off. In particular, when precision is fixed at \(0.8\), recall remains low and non-monotonic across model sizes. This shows that larger models do not reliably recover more valid continuations under a high-precision constraint. Therefore, while scale slightly improves order calibration, it does not by itself resolve the failure in order calibration.

\paragraph{Oracle filtering.} To quantify how much diversity is lost specifically due to order miscalibration, we present an oracle validity filter that samples only from tokens labeled valid by our diagnostic procedure. We apply the oracle filter only during the first two decoding steps and then continue generation normally. For other strategies, we sweep a grid of temperatures and cutoff parameters and report the best diversity for validity at least $0.8$ (see Appendix~\ref{app:oracle-cutoff}).

Table~\ref{tab:results} shows that even this limited oracle intervention yields a clear improvement: oracle filtering achieves the highest embedding diversity and the lowest Self-BLEU. This indicates that valid tokens excluded by rank-based cutoffs yield meaningfully different generations. Therefore, the order calibration gap has direct output-level consequences: valid alternatives are present in the model distribution, but standard rank-based sampling rules fail to reliably expose them.

\section{Shape Calibration Fails: Sharp, heavy-tailed distributions limit diversity}
\label{sec:strong_calibration}

In this section, we focus on a second, complementary bottleneck, namely shape calibration. While the LLM conditional distribution varies across tasks, prior work suggests that next-token distributions are typically head-heavy and long-tailed, with most probability mass concentrated in a relatively small nucleus and a large, unreliable tail~\citep{truncation, Holtzman2020The}. In Appendix~\ref{app:logits_fit}, we randomly sample conditional next-token distributions from a diverse set of tasks and examine their sorted logits. We show that the logits' behavior can be consistently described by a linear decay in the head, followed by a heavy-tailed distribution that decays logarithmically. After applying softmax, this translates to an exponential (geometric) decay in the head, i.e., $p(v_k\mid y_{<t}) \propto \exp(-\lambda k/T)$ and a \emph{Zipf-like}~\citep{basu2021mirostat} behavior in the tail, $p(v_k\mid y_{<t}) \propto k^{-\lambda/T}$, where $k$ denotes the rank of token $v_k$ in the sorted vocabulary, $T$ is the temperature and $\lambda$ controls the sharpness of the distribution.

These two properties together cause the shape calibration issue: the distribution is always very sharp over a very small portion of the head, even in tasks where we expect to observe an exact uniform distribution on valid tokens. Temperature scaling is therefore often used to flatten the distribution head, increasing the probability of valid regions outside the head. However, temperature scaling comes with a necessary caveat. Although each invalid token remains with a small probability, the accumulated probability of invalid tokens especially grows very quickly with temperature scaling, leading to an unwanted validity--diversity trade-off.

While understanding how distributional miscalibration arises from LLM training and design pipeline is an important question~\citep{finlayson2024closing, chang-mccallum-2022-softmax}, we focus on its implications for the diversity-validity trade-off. We emphasize that characterizing the exact geometric form is not intended as a literal empirical claim, but as an analytically convenient proxy for heavy-tailed distributions.

\subsection{How severely does distribution shape affect validity--diversity trade-off?}
\label{sec:temp_trade-off}
Although temperature is the most basic way of injecting randomness into decoding, in practice, higher temperatures often fail to recover a broad set of valid outputs. 
Our goal in this subsection is to characterize this limitation by attributing the failure to distribution sharpness. In fact, we show that even if the order calibration disappears, sharp LLM conditionals induce heavy validity--diversity trade-offs. For simplicity of the proof, we further impose the following assumption: % on the valid set size at each step.
\begin{assumption}[Invariant valid branching]
\label{assump:invariant_valid_branching}
Assume that all valid sequences for a task have a fixed length $d$. For every valid prefix \(y_{<t}\), the number of valid next-token choices depends only on the position \(t\), not on the particular prefix. That is, there exist integers \(v_1,\ldots,v_d\) such that for every valid prefix $y_{<t}$,
\[
    |G(y_{<t})| = v_t.
\]
Furthermore, we define the branching length as the number of positions at which there is more than one valid continuation:
\[
    m
    :=
    \left|\{t\in[d]: v_t\ge 2\}\right|.
\]
\end{assumption}
Assumption~\ref{assump:invariant_valid_branching} fixes the generation length and removes prefix-level heterogeneity, allowing the diversity loss to be expressed in terms of the effective branching length $m$.

\begin{theorem}[Validity--Diversity trade-off]
\label{thm:length}
Consider a length-\(d\) generation task with a valid set \(V\). Suppose that, at each valid prefix, the model's ranked next-token distribution is geometrically decaying, and suppose the valid next tokens occupy the top ranks. Then any temperature-scaled distribution satisfying
\[
    \mathrm{Val}(p)\ge 1-\epsilon
\]
also satisfies
\[
    \mathrm{Div}(p)
    \le
    e^{-m c(\epsilon)},
\]

for some positive constant \(c(\epsilon)>0\), where $c(\epsilon)\to \ln 2$ as $\epsilon\to0$.
\end{theorem}
\paragraph{Interpretation.}
The theorem shows that temperature scaling pays a local diversity price at every branching step. To achieve high validity, the distribution must be sharp enough that invalid tokens receive little mass. At the same time, this sharpness makes the conditional distribution over valid tokens non-uniform, concentrating probability on the highest-ranked valid continuations. These per-step entropy losses add across the sequence; after exponentiating entropy to obtain diversity, they yield an exponential decay in $m$. Moreover, the rate \(c(\epsilon)\) increases as the validity requirement becomes stricter. In the high-validity regime, the bound approaches $m$, and can be stronger when many steps contain several valid continuations. We provide the formal proof in Appendix~\ref{app:length_proof}.
\begin{figure}[t]
    \centering
      \begin{subfigure}[!t]{0.41\textwidth}
        \includegraphics[width=1.0\linewidth]{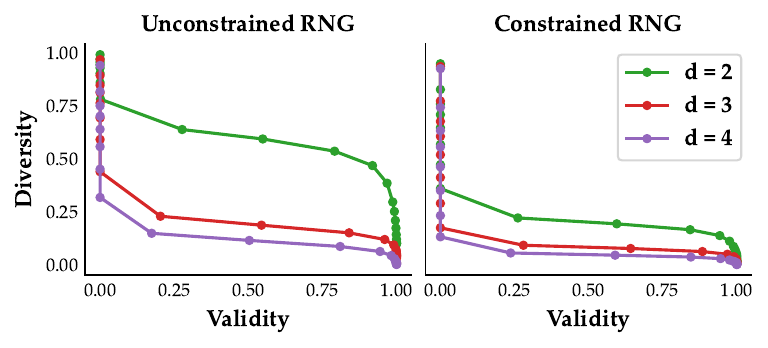}
      \end{subfigure}
       \hfill
    {\color{gray}\vrule width 0.2pt}
    \hfill
      \begin{subfigure}[!t]{0.52\textwidth}
      \includegraphics[width=\linewidth]{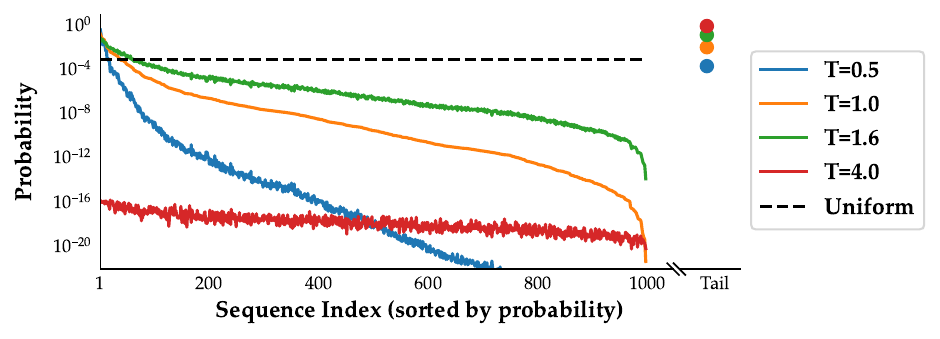}
      \end{subfigure}
    \caption{Effects of temperature scaling in random number generation on \texttt{Olmo-3-7B-Instruct}. \textbf{(Left)} Validity--diversity trade-offs for unconstrained and constrained random-number generation tasks across sequence lengths \(d\). Longer generations exhibit a stronger validity--diversity trade-off. \textbf{(Right)} Valid sequence probabilities for the length-\(3\) unconstrained random-number generation task. }
    \label{fig:temp_scaling}
\end{figure}
\subsection{Empirical Investigation}
\paragraph{Controlled random generation testbed.} Investigating the diversity--validity trade-off as a function of distribution shape requires us to calculate the probabilities of every valid sequence. Sweeping all the conditionals on open-ended generation is infeasible. To address this, we propose two tasks with known valid sets as controlled testbeds: random-number generation and naming a random US state. We consider two variants of the random number generation task in Figure~\ref{fig:temp_scaling}. In the \emph{unconstrained} setting, the model is asked to generate a length-\(d\) sequence of i.i.d. digits, where each digit lies in \(\{0,\dots,9\}\). Thus, every sequence in \(\{0,\dots,9\}^d\) is valid. We expect the model to impose a uniform distribution on each conditional, given independence. In the \emph{constrained} setting, the model is asked to generate a length-\(d\) sequence of digits whose sum equals a specified target \(N\). Appendix~\ref{app:shape_calibration} provides further details, including exact prompts and our experimental setup.

These tasks offer several key advantages. First, we can compute the exact generation probability for each valid sequence, since the entire valid set is known. Second, we objectively expect each valid sequence to have the same generation probability. Thus, any deviation from uniform distribution leads to a systematic validity--diversity trade-off. Moreover, the unconstrained random-number generation task satisfies Assumption~\ref{assump:invariant_valid_branching}, serving as a proper testbed for our theorem.

\paragraph{Sequence validity--diversity trade-offs.} For each candidate sequence in $V$, we therefore feed the corresponding prefixes into the LLM and extract the next-token logits. We then apply different temperatures to these logits and compute, exactly over the known sequence space, both the probability mass assigned to the valid region and the entropy of the model's distribution conditioned on validity. This gives the validity--diversity curve induced by temperature scaling alone.

The left panel shows the validity--diversity frontier for random-number generation across sequence lengths \(d\in\{2,3,4\}\). As the length increases, the frontier becomes sharper: maintaining high validity requires a larger reduction in diversity. This is consistent with Theorem~\ref{thm:length}, where each branching position contributes a local entropy loss and these losses compound across the sequence. The constrained setting exhibits an even stronger trade-off, despite violating Assumption~\ref{assump:invariant_valid_branching}, suggesting the broader applicability of our result.
%This finding is particularly surprising in the unconstrained case: every next random number should be independent from the last, hence, it is surprising to see worsened trade-off with increasing sequence length. 

\paragraph{Shape calibration in sequence-level.}
Although we have studied the shape of the conditional distribution, its implications at the sequence level are less transparent. The right panel of Figure~\ref{fig:temp_scaling} provides a sequence-level view of temperature scaling trade-off. We plot the probabilities of each valid sequence in the unconstrained length-\(3\) task, and show the magnitude of the invalid sequence mass. We observe that the distribution is already highly concentrated: a small number of sequences receive orders-of-magnitude larger probability than the rest. At the same time, many sequences lie in a long, low-probability tail. Raising the temperature can indeed move probability mass toward this tail, thereby increasing diversity among rare valid outputs. However, the invalid sequence mass dominates even more as temperature increases. Therefore, shifting mass toward the tail improves diversity only at the cost of reduced validity. Thus, the empirical behavior mirrors the theory: temperature can flatten the distribution, but it cannot selectively recover valid diversity.

\subsection{When shape and order miscalibration interact}
\label{sec:top_token_filtering}

Top-token filtering is often applied after temperature scaling to suppress invalid tail mass. However, this does not remove the calibration problem; it couples shape calibration with order calibration. A cutoff rule retains a prefix of the ranked distribution, so to preserve valid diversity while maintaining validity, this prefix must approximate the valid-token set \(G(y_{<t})\). This requires two conditions: valid tokens must be concentrated near the top of the ranked distribution, and the cutoff rule must adapt to the local boundary of the valid set. Different methods encode different boundary assumptions: top-\(k\) assumes a roughly fixed support size, top-\(p\) assumes a stable cumulative-mass boundary, and min-\(p\) assumes a stable relative-probability gap from the top token.

\paragraph{Comparing validity--diversity trade-offs.} We test these assumptions on controlled tasks, including random number generation and random state generation tasks, where the valid set is known exactly. Figure~\ref{fig:state_top_token_filtering} reports the trade-offs on both tasks, and more results are found in Appendix~\ref{app:shape_calibration}. For each cutoff strategy, we sweep both temperature and the method-associated parameter, since different methods can achieve their best validity--diversity trade-off at different temperatures~\citep{zhou2025balancingdiversityriskllm}. Comparing methods at a single fixed temperature can therefore be misleading: poor performance may reflect a bad choice of temperature rather than a limitation of the filtering rule itself.

\paragraph{Cutoff Oracle Strategy.} To study the miscalibration, we include a cutoff oracle strategy in Figure~\ref{fig:state_top_token_filtering}. Given that our generation task has a known ground truth, we directly compute the valid size at each generation step. At each step, the oracle cutoff rule knows only the number of valid tokens,
\(
    g_t^\star = |G(y_{<t})|,
\)
and retains the top \(g_t^\star\) ranked tokens. However, the oracle does \textit{not} know which tokens are valid, and therefore remains a rank-based cutoff rule. On the random state generation task, the cutoff oracle achieves the ideal point $(1.0, 1.0)$, indicating that the model is order-calibrated in this controlled setting. Therefore, the remaining gap reflects a failure of the methods' implicit shape assumptions: fixed top-$k$, cumulative-mass, or relative-probability thresholds do not identify the valid-token boundary. On the random-number generation task, the cutoff oracle also does not immediately achieve high diversity, showing that shape miscalibration is coupled with order miscalibration. 

\begin{figure}[!t]
  \centering
  \begin{minipage}[t]{0.55\linewidth}
    \vspace{0pt}
    \centering
    \includegraphics[
      width=\linewidth,
      trim={0 2mm 0 0},
      clip
    ]{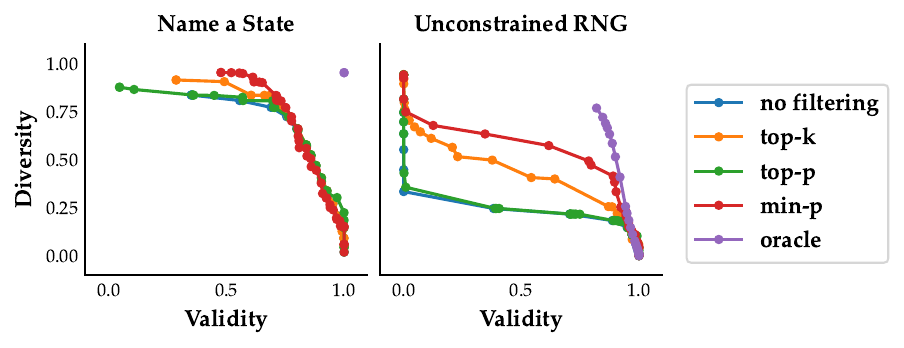}
  \end{minipage}
  \hfill
  \begin{minipage}[t]{0.44\linewidth}
    \vspace{0pt}
    \captionsetup{
      justification=justified,
      singlelinecheck=false,
      font=small
    }
    \caption{
    Validity--diversity Pareto frontiers for top-token filtering methods on generation tasks from \texttt{Llama-3.1-8B-Instruct}. \textbf{(Left)} Name a random state in the US. \textbf{(Right)} Unconstrained random number. Each sampling method is swept over both temperature and its own truncation parameter. The oracle-size cutoff retains the top \(|G(y_{<t})|\) tokens at each prefix.}
    \label{fig:state_top_token_filtering}
  \end{minipage}
\end{figure}

\section{Future Work}
This paper attributes the LLM validity--diversity trade-off to two distributional properties: order and shape calibration. Through empirical demonstrations and theoretical analysis, we show that miscalibration is a recurring bottleneck across model families, sizes, and controlled generation tasks. Finally, as a preliminary piece of future work, in Appendix~\ref{app:coding} we observe that order and shape calibration can potentially provide insights into domains beyond open-ended diversity generation.

Several directions remain open. First, our results suggest that future decoding methods should move beyond top-token filtering rules. Calibration-aware decoders could instead incorporate auxiliary validity signals. Second, it remains important to understand where these calibration failures come from. Pretraining, instruction tuning, preference optimization, and safety
alignment may each affect the sharpness of the distribution and the rank ordering of valid alternatives. Studying these effects could suggest training objectives that preserve broader valid support without sacrificing quality.
%%%%%%%%%%%%%%%%%%%%%%%%%%%%%%%%%%%%%%%%%%%%%%%%%%%%%%%%%%%%

% \begin{ack}
% Use unnumbered first level headings for the acknowledgments. All acknowledgments

% Do {\bf not} include this section in the anonymized submission, only in the final paper. You can use the \texttt{ack} environment provided in the style file to automatically hide this section in the anonymized submission.
% \end{ack}

\newpage
\bibliographystyle{plainnat}   % common choice
\bibliography{main}      % references.bib

% %%%%%%%%%%%%%%%%%%%%%%%%%%%%%%%%%%%%%%%%%%%%%%%%%%%%%%%%%%%%
\newpage 

\appendix
\section{LLM-as-a-judge Details}
\label{app:llm-as-a-judge}
\subsection{Prompt and Model}
\label{app:llm-as-a-judge-prompts}
Token validity is scored based on the greedy-decoding completion. We use \texttt{Qwen3.5-35B-A3B-FP8} with thinking enabled as our judge. The prompts used are documented below. 

\begin{promptbox}{LLM-as-a-judge Prompts}
(*@{\normalfont\bfseries System Prompt}@*) 
You are a strict language quality judge. You will first see a question and a generated answer. Your task is to evaluate the quality of the generated answer.

(*@{\normalfont\bfseries Evaluation Instructions}@*)
Evaluate the quality of the generated answer based on the following criteria:
1) grammar, spelling and punctuation,
2) semantic soundness, validity, and relevance to the question,
3) overall quality.

When evaluating grammar, check for spelling mistakes, punctuation errors, and grammatical issues. If spaces are missing between words, extra punctuations in the middle of sentences, or incorrect capitalization, that should be considered a grammar error. 

Additionally, if the generation contains non-English characters, that should be considered a grammar error and scored low on grammar, even if the rest of the generation is grammatically correct.

When evaluating semantics, check if the answer is relevant to the question, if it makes sense, and if it is factually correct. An answer that is fluent but does not make sense or is not relevant to the question should be scored low on semantics.

Be concise and nitpicky in your evaluation. Short responses are acceptable, as long as they are grammatically correct and semantically sound.

First provide an explanation of your evaluation, then give a score for each category. Score each category on a scale of 1 to 10, where 1 is very poor and 10 is excellent.

Any non-English text in the generation should be considered a grammar error and scored low on grammar, even if it's semantically correct.

The generation does not need to be fully finished to be considered as valid. Only look at the part of the generation that is present and evaluate that. If the generation is cut off in the middle of a sentence, evaluate the part that is present and ignore the fact that it is cut off.

Return only valid JSON with this exact schema:
{
  "reason": <brief explanation>,
  "grammar": <int 1-10>,
  "semantic": <int 1-10>,
  "overall": <int 1-10>
}

Here is the question:
```
{question}
'''

And here is the generated answer to evaluate:
```
{shot}
'''

Your evaluation:
\end{promptbox}

\subsection{Evaluating LLM-as-a-judge}
\label{app:llm-as-a-judge-results}
To assess the inter-annotator reliability, three authors independently annotated 100 generations using the rubric in Appendix~\ref{app:llm-as-a-judge-prompts}. 

We compute inter-annotator agreement on the overall validity scores using Krippendorff's alpha~\citep{Krippendorff1980ContentAA}, obtaining $\alpha = 0.759$. This indicates reasonably strong agreement among annotators, suggesting that the rubric yields consistent human judgments.

To contrast our LLM-based evaluation against human judgment, we compared \texttt{Qwen3.5-35B-A3B-FP8} predictions against human labels from different annotators for 100 generations. We use a threshold of $9$ on grammar, semantic and overall score. We treat the judge model score above this threshold as a positive prediction (valid token) and scores below the threshold as negative predictions (invalid tokens). Under this rule, \texttt{Qwen3.5-35B-A3B-FP8} achieves $80.90$\% accuracy and $79.60$\% F1 when compared to humans. We also compare \texttt{GPT-5.4} against \texttt{Qwen3.5-35B-A3B-FP8} under the same rule, obtaining $73.40$\% accuracy and $61.45$\% F1. These results suggest that our LLM-based evaluation can approximate human judgment and \texttt{GPT-5.4}.

\subsection{Greedy Decoding Robustness}
\label{app:greedy_vs_sampling}
Our token-validity annotation procedure uses greedy decoding after forcing a candidate token, treating the validity of the resulting completion as a proxy for whether that token preserves access to a valid continuation. One concern is that a token deemed invalid under greedy decoding may still admit a valid sampled continuation, and vice versa. To test this, Figure~\ref{fig:sampling_analysis} reports an ablation in which we replace the greedy continuation with model suggested sampling parameters when constructing token-level validity labels. We sample 10 continuations for each token, and then recompute the resulting precision--recall curves under the same cutoff strategies. The curves are nearly unchanged, indicating that the estimated precision and recall are not driven by the particular choice of greedy decoding. This supports using greedy continuation as a computationally efficient proxy for token validity in the main experiments.

\begin{figure}[ht]
    \centering
    \includegraphics[width=0.5\linewidth]{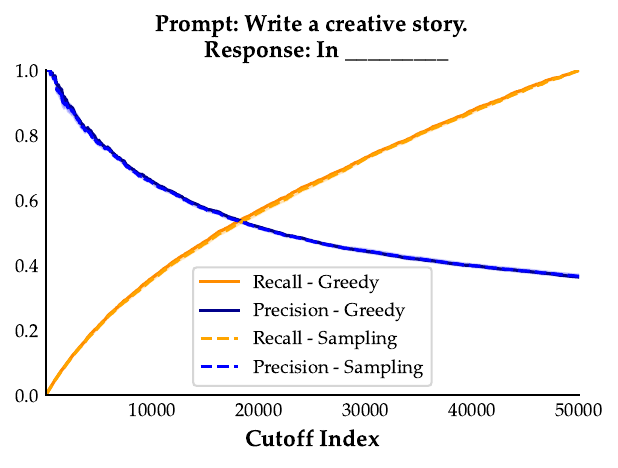}
    \caption{\textbf{Robustness of token-validity estimates to the continuation procedure.} We compare local precision--recall curves computed from token labels obtained by greedy decoding after each forced candidate token with labels obtained from stochastic sampling. The resulting curves nearly overlap across cutoff indices, indicating that our estimated precision--recall trade-off is not sensitive to using greedy decoding as the continuation procedure.}
    \label{fig:sampling_analysis}
\end{figure}
\section{Order Calibration Trade-offs}
\label{app:pr}

\subsection{Generation Setup}
\label{app:pr-setup}

We use a variety of LLMs in our experiments, each documented in the subsection below. We randomly select 5 open-ended generation categories from NoveltyBench, ranging from story-telling, joke-telling, poems, and item selection \citep{zhang2025noveltybench}. 

% \begin{promptbox}{Open-ended Generation Prompts}
% 1. Write a brief creative story.
% 2. Write a short love poem with 4 lines.
% 3. Tell me a dad joke.
% 4. What is the top item you would add to a grocery list for a memorable shopping experience?
% \end{promptbox}
\subsection{Sweep Generation Experiment}
\label{app:pr-examples}
Experiments were repeated 10 times for reproducibility. The cost of exhaustive sweeping grows as $N^d$, where $d$ is the sweeping depth and $N$ is the number of candidate tokens evaluated at each conditional. To keep the oracle evaluation computationally tractable, we sweep only three consecutive conditionals. At each conditional, we consider candidates up to rank $1000$ and subsample every tenth token, yielding $N=100$ candidates per conditional and $100^3$ evaluated branches in total. 

In Table~\ref{tab:story-generation-examples}, we provide examples of tokens swept in the story generation task. We sweep up to depth 3 from randomly selected starting token positions. All examples are generated with \texttt{Qwen3.5-35B}.

\begin{table}[ht]
\centering
\caption{
Generation examples from the creative story generation task. The sweeping is performed on the second token after ``In''. For each candidate token, we force the model to continue from the prefix with that token, greedily decode
the remaining sequence, and evaluate the resulting completion with the judge.
}
\label{tab:story-generation-examples}
\small
\setlength{\tabcolsep}{3.5pt}
\renewcommand{\arraystretch}{1.15}
\begin{tabular}{c c p{0.62\linewidth} c}
\toprule
Index & Token & Generated sequence & Judge Score \\
\midrule
1  & \textcolor{validgreen}{\texttt{2045}} & \textbf{In 2045}, the last library was a single glass pod floating above the clouds, where an old librarian read stories to a child who had never seen a paper book. & 10 \\
43  & \textcolor{invalidred}{\texttt{edible}} & \textbf{Inedible}, the clock on the wall began to tick backwards, unspooling the day until the coffee cup in my hand was whole again, the steam rising into the past. & 8 \\
201  & \textcolor{validgreen}{\texttt{Obsidian}} & \textbf{In Obsidian} City, the streetlamps were made of captured starlight, and the baker sold loaves of warm silence to quiet the noisy crowds. One rainy Tuesday, a little girl bought a loaf, but when she bit into it, she heard the sound of a thousand birds singing in a language she had never known. She smiled, realizing the city wasn't just quiet; it was waiting for her to listen.  & 10 \\
208  & \textcolor{invalidred}{\texttt{eeded}} & \textbf{Ineeded} to write a story, so I did. & 7 \\
507  & \textcolor{validgreen}{\texttt{ebriated}} & \textbf{Inebriated} by the scent of rain on hot asphalt, the old streetlamp blinked its last, dreaming of a sky it would never see again. & 10\\
508  & \textcolor{invalidred}{\texttt{\zh{忘却}}} & \textbf{In\zh{忘却}} of a forgotten clock tower, a tiny mouse named Pip found a gear that could rewind time by one second. Every day, he used it to fix his mistakes, until he realized he had no memories of the past, only the endless, perfect present. & 5\\
10002  & \textcolor{validgreen}{\texttt{Eco-Station}} & \textbf{In Eco-Station} 42, the last human, Elara, watered a single, glowing fern that hummed a lullaby to the empty room. & 10\\
23606  & \textcolor{invalidred}{\texttt{gone}} & \textbf{In gone}, the clock struck thirteen, and the moon turned into a giant, glowing cheese wheel that the city's cats began to chase, leaving trails of stardust in their wake. & 7\\
35007  & \textcolor{validgreen}{\texttt{Darth}} & \textbf{In Darth Vader}'s castle, a tiny droid named BEEP-7 found a single, glowing blue flower growing in a crack of the floor. He spent his entire existence protecting it from the stormtroopers, watering it with his own coolant, until the day the Emperor arrived. BEEP-7 stood before the dark lord, not with a weapon, but with the flower, and for a moment, the dark side of the Force faltered, remembering a time before the fall. & 10\\
42101  & \textcolor{invalidred}{\faSnowflake} & In \faSnowflake\, the quietest corner of the world, a tiny snowflake named Pippin decided to dance instead of fall, turning the gray sky into a swirling ballroom of silver light.
 & 8\\
42273 & \textcolor{validgreen}{\texttt{cartoons}} & \textbf{In cartoons}, the sky is always blue, but in the real world, it is a canvas of shifting grays and golds. One day, a little girl named Elara found a paintbrush that could touch the sky. She dipped it in the color of her favorite sunset and painted a dragon made of fire across the clouds. The dragon roared, not with sound, but with warmth, and the world below smiled as the cold winter faded into an eternal spring. & 10 \\
\bottomrule
\end{tabular}
\end{table}

\subsection{Model Size Experiments}
\label{app:pr-model-size}

\begin{figure}[ht]
    \centering
    \includegraphics[width=0.7\linewidth]{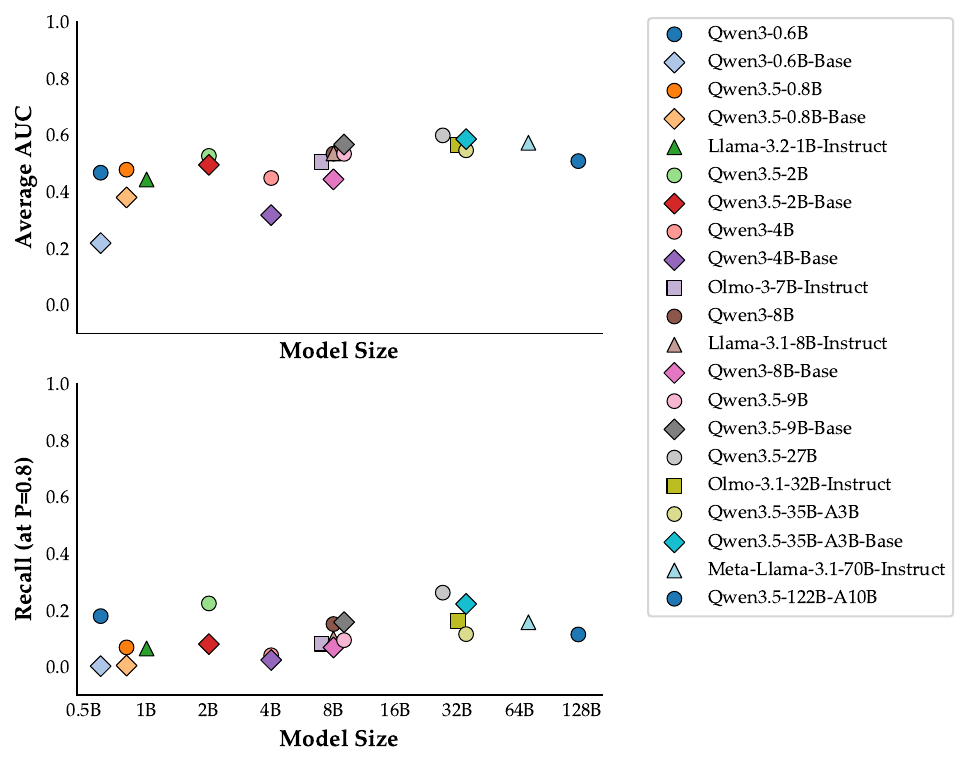}
    \caption{Detailed view of Figure~\ref{fig:model_size}. We evaluate models across 3 families with various sizes. Diamond represents pre-trained only models; circle represents post-trained \texttt{Qwen} models; square represents post-trained \texttt{Olmo} models; triangle represents post-trained \texttt{Llama} models.}
    \label{fig:model_size_ablations}
\end{figure}

Figure~\ref{fig:model_size_ablations} provides a detailed view of the models used for our experiment. 

\subsection{Oracle Sampling Experiments}
\label{app:oracle-cutoff}
For all our oracle generations, we only sweep the first two decoding steps. We sweep the first conditional to obtain 1000 valid tokens. We then uniformly sample from these 1000 tokens to obtain the second-step conditionals and repeat the sweep once.

In Table~\ref{tab:oracle_sweep}, we report the parameter ranges we swept for each cutoff strategy. For each parameter, we sample 1000 generations to compute the semantic and lexical diversity. While there are multiple methods for measuring semantic and lexical diversity, we select embedding diversity and Self-BLEU as representatives.

\begin{table}[ht]
\centering
\small
\caption{
Parameter grids used for decoding sweeps. For each parameter setting, we sample
$1000$ generations to report semantic and lexical diversity.
}
\label{tab:sweep_ranges}
\begin{tabular}{ccc}
\toprule
\textbf{Decoding strategy} & \textbf{Parameter} & \textbf{Values swept} \\
\midrule
Temperature sampling
& $T$
& $\{0.01, 0.3, 0.6, 0.8, 1.0, 1.2,$ \\
& & $\phantom{\{}1.4, 1.6, 1.8, 2.0, 2.5, 3.0\}$ \\

top-$k$
& $k$
& $\{10, 20, 50, 80, 100, 500\}$ \\

top-$p$
& $p$
& $\{0.1, 0.5, 0.7, 0.9, 0.95\}$ \\

min-$p$
& $p_{\min}$
& $\{0.01, 0.1, 0.5, 0.9\}$ \\
\bottomrule
\end{tabular}
\label{tab:oracle_sweep}
\end{table}

\paragraph{Embedding Diversity}
For each generated sequence, we obtain an embedding vector \(\mathbf e_i\) from \texttt{Qwen3-Embedding-8B}. We compute pairwise cosine distances and define embedding diversity as 
\[
    \mathrm{Embedding\ Diversity}
    =
    \frac{2}{n(n-1)}
    \sum_{i<j}  (1 - \cos(\mathbf e_i, \mathbf e_j)).
\]
The range is between $[0,1]$, with higher values indicating higher semantic diversity.

\paragraph{Self-BLEU \citep{alihosseini-etal-2019-jointly}}
Given generations $\{y_1,\dots,y_n\}$ for a single task, we compute
\[
    \mathrm{Self\text{-}BLEU}
    =
    \frac{1}{n}
    \sum_{i=1}^{n}
    \mathrm{BLEU}\!\left(y_i, \{y_j\}_{j\ne i}\right)
    \in [0,1].
\]
where BLEU \citep{papineni-etal-2002-bleu} measures the \(n\)-gram overlap
between a candidate output and a set of reference outputs. In Self-BLEU, lower values indicate higher lexical diversity. We use $n=4$ in our experiments.

\section{Shape Calibration Trade-offs}
\label{app:shape_calibration}
\subsection{Random Number Generation}

\label{app:rng}
\begin{promptbox}{Random Number Generation Prompts}
(*@{\normalfont\bfseries Unconstrained Generation}@*) 
Generate {n} random integers between {start} and {end} (inclusive). Your response should be {n} integers separated by commas and no white spaces. Answer:

(*@{\normalfont\bfseries Constrained Generation}@*) 
Generate {n} random integers between {start} and {end} (inclusive) that sum to at most {sum_value}. Your response should be {n} integers separated by commas and no white spaces. Answer:
\end{promptbox}

\begin{figure}[ht]
    \centering
    \includegraphics[width=0.9\linewidth]{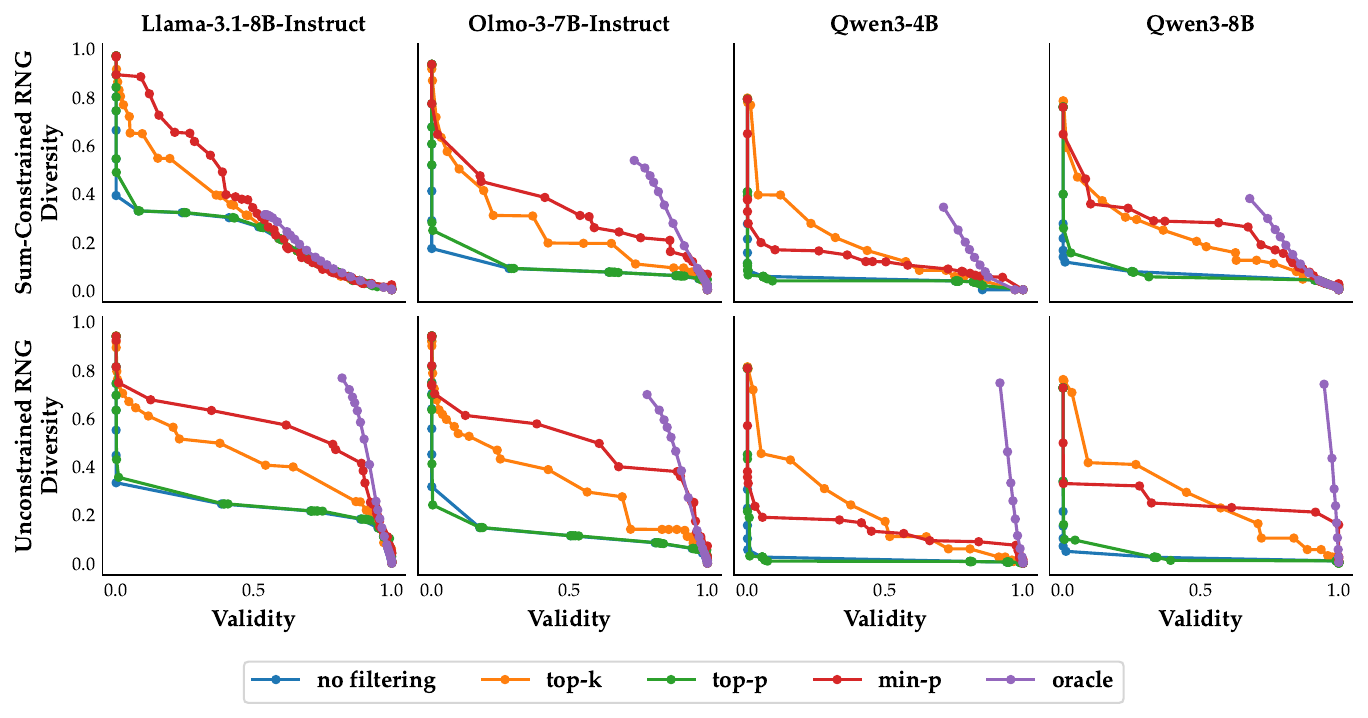}
    \caption{Validity--Diversity trade-off frontiers for constrained and unconstrained random number generation across 4 model families.}
    \label{fig:full_dv_random}
\end{figure}

Figure~\ref{fig:full_dv_random} shows the validity--diversity trade-off frontiers across all four model families. The cutoff oracle helps disentangle the effects of shape and order calibration. In the unconstrained setting, the oracle often approaches the ideal point $(1.0, 1.0)$, suggesting that the main gap stems from standard top-token filters' inability to infer the correct support size. In contrast, in the sum-constrained setting, the oracle itself remains separated from $(1.0, 1.0)$, indicating stronger order miscalibration. The additional gap between standard top-token filtering methods and the oracle reflects shape miscalibration, since fixed top-$k$, cumulative-mass, and relative-probability thresholds do not reliably recover the correct valid-token boundary. Together, these results show that shape and order miscalibration jointly contribute to the validity--diversity trade-off, with their interaction becoming more pronounced under compositional constraints.

\subsection{Random State Generation}
\label{app:random_state}
\begin{promptbox}{Name a Random State Prompt}
Randomly name a US state with no additional explanation. Answer:    
\end{promptbox}

Figure~\ref{fig:state_seq} shows sequence-level probabilities under different temperatures on the random state generation task. As in Figure~\ref{fig:temp_scaling}, the distribution is highly concentrated and heavy-tailed: a small number of valid sequences receive much larger probability than the rest. Increasing temperature flattens the valid-sequence distribution, but it also shifts substantial probability mass into the invalid region before the model approaches a uniform distribution over valid states. Thus, temperature scaling again improves valid diversity only by sacrificing validity.

\begin{figure}[ht]
    \centering
    \includegraphics[width=0.7\linewidth]{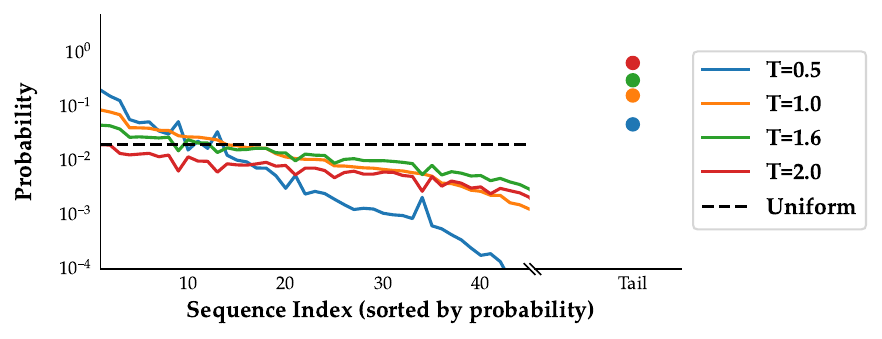}
    \caption{Sequence probability for random state generation task. Sequences are sorted by probability, and probabilities are plotted in log-space. The ``tail'' section represents the total probability mass of the invalid region.}
    \label{fig:state_seq}
\end{figure}

\section{Formal Definitions of Calibration}

A decoding rule achieves high validity if it assigns high probability only on valid tokens at each time step and it achieves high diversity if it explores many distinct tokens in $G$ rather than concentrating on only a few of them. Therefore, we define order and shape calibration as:

\begin{definition}[Calibration]
Given $y_{<t} \in \mathcal{V}^{t-1}$ and an LLM conditional distribution $p$,

\begin{enumerate}[left=1ex, topsep=1pt, itemsep=1pt, parsep=1pt]
    \item \textbf{(Order Calibration)} $p$ is order calibrated if for any valid token $v\in\mathcal{V}$ and invalid token $w\in\mathcal{V}$, it assigns a higher probability to the valid token, i.e., $p(v\mid y_{<t}) \geq p(w\mid y_{<t})$.
    
    \item \textbf{(Shape Calibration)} $p$ is shape calibrated if for any token $v\in \mathcal{V}$, it assigns probability mass to $v$ according to the number of valid continuations starting with $v$, i.e., $p(v\mid y_{<t}) \propto N(y_{<t} \circ v)$.
\end{enumerate}
\end{definition}

Note that shape calibration is stronger than order calibration: even if order calibration is resolved, shape calibration can still persist. However, perfect shape calibration is strictly harder than order calibration.

Our notation of ``calibration'' is not directly related to LLM confidence calibration literature. 

\section{Analysis of Theorem~\ref{thm:mult}}
\label{app:mult}

In this section, we provide a formal analysis of how local truncation decisions
affect sequence-level behavior. We restate the key definitions for completeness
and give full proofs of the results in Section~\ref{sec:weak_calibration}.

\subsection{Setup and notation}

Fix a valid set $V \subseteq \mathcal{V}^d$. At each decoding step $t$, a rule
retains a nonempty subset $S_t(y_{<t}) \subseteq \mathcal{V}$ and samples
\[
Y_t \sim \mathrm{Unif}(S_t(Y_{<t})).
\]
Let $Q_S$ denote the induced distribution over full sequences.

Let $G(y_{<t})$ denote the set of valid next tokens, and let
$N(y_{<t})$ denote the number of valid completions extending $y_{<t}$.

\paragraph{Local precision and recall.}
For a prefix $y_{<t}$, define
\[
\mathrm{Prec}_t(S;y_{<t})
:=
\frac{|S_t(y_{<t}) \cap G(y_{<t})|}{|S_t(y_{<t})|},
\]
and
\[
\mathrm{Rec}_t(S;y_{<t})
:=
\frac{
\sum_{v \in S_t(y_{<t}) \cap G(y_{<t})}
N(y_{<t}\circ v)
}{
N(y_{<t})
}.
\]

\paragraph{Sequence-level precision and recall.}
\[
\mathrm{Prec}_{\mathrm{seq}}(S)
=
Q_S(Y \in V),
\qquad
\mathrm{Rec}_{\mathrm{seq}}(S)
=
\frac{
|\{y \in V : y_t \in S_t(y_{<t}) \ \forall t\}|
}{|V|}.
\]

\subsection{Multiplicative decomposition}

\begin{theorem}[Exact multiplicative decomposition]
\label{thm:mult-decomp-app}
\[
\mathrm{Prec}_{\mathrm{seq}}(S)
=
\prod_{t=1}^d \alpha_t(S),
\qquad
\mathrm{Rec}_{\mathrm{seq}}(S)
=
\prod_{t=1}^d \beta_t(S),
\]
where
\[
\alpha_t(S)
:=
\mathbb{E}_{Q_S}\!\left[
\mathrm{Prec}_t(S;Y_{<t})
\mid Y_s \in G(Y_{<s}) \ \forall s<t
\right],
\]
and
\[
\beta_t(S)
:=
\mathbb{E}_{U_V}\!\left[
\mathrm{Rec}_t(S;Y_{<t})
\mid Y_s \in S_s(Y_{<s}) \ \forall s<t
\right].
\]
\end{theorem}

\begin{proof}
Define
\[
F_t := \{Y_s \in G(Y_{<s}) \text{ for all } s \le t\}.
\]
A sequence is valid if and only if it preserves access to a valid continuation
at every step, hence $\{Y \in V\} = F_d$. By the chain rule,
\[
\mathrm{Prec}_{\mathrm{seq}}(S)
=
Q_S(F_d)
=
\prod_{t=1}^d Q_S(F_t \mid F_{t-1}).
\]
Conditioned on $F_{t-1}$ and $Y_{<t}$, the next token is sampled uniformly
from $S_t(y_{<t})$, yielding
\[
Q_S(F_t \mid F_{t-1}, Y_{<t})
=
\mathrm{Prec}_t(S;y_{<t}).
\]
Taking expectations gives $\alpha_t(S)$.

For recall, let $U_V$ denote the uniform distribution over $V$ and define
\[
E_t := \{Y_s \in S_s(Y_{<s}) \text{ for all } s \le t\}.
\]
Then
\[
\mathrm{Rec}_{\mathrm{seq}}(S)
=
U_V(E_d)
=
\prod_{t=1}^d U_V(E_t \mid E_{t-1}).
\]
Conditioned on $Y_{<t}$, the next token under $U_V$ is distributed
proportionally to continuation counts, yielding $\mathrm{Rec}_t(S;y_{<t})$.
Taking expectations gives $\beta_t(S)$.
\end{proof}

\paragraph{Local trade-off view.}
The multiplicative decomposition shows that sequence-level precision and recall
are governed by accumulated local log-losses,
\[
    u_t(S):=-\log \alpha_t(S),
    \qquad
    v_t(S):=-\log \beta_t(S).
\]
Thus, maintaining high sequence-level precision imposes a small total budget on
the precision losses $\sum_t u_t(S)$. Theorem~\ref{thm:formal-compounding}
formalizes the consequence: if many steps necessarily incur nontrivial recall
loss whenever their precision loss is small, then sequence-level recall must
decay exponentially.

\subsection{Compounding effect of hard decoding steps}

\paragraph{Formalization of Theorem~\ref{thm:mult}.}
The main text theorem states that if many decoding steps incur an unavoidable
loss in recall when enforcing high precision, then sequence-level recall
decays exponentially. We now formalize this statement.

\begin{definition}[$(\eta,\rho)$-hard step]
Fix $\eta,\rho>0$. A step $t$ is $(\eta,\rho)$-hard if, for every decoding rule $S$,
\[
    u_t(S)\le \eta
    \quad \Longrightarrow \quad
    v_t(S)\ge \rho .
\]
Equivalently, whenever the local precision loss at step $t$ is at most $\eta$, the local recall loss is at least $\rho$.
\end{definition}

\begin{theorem}[Compounding effect of hard decoding steps]
\label{thm:formal-compounding}
Suppose at least $m$ decoding steps are $(\eta,\rho)$-hard. Then any decoding rule satisfying
\[
    \mathrm{Prec}_{\mathrm{seq}}(S)\ge 1-\delta
\]
also satisfies
\[
    \mathrm{Rec}_{\mathrm{seq}}(S)
    \le
    \exp\!\left(
    -\rho
    \left(
    m-\frac{-\log(1-\delta)}{\eta}
    \right)_+
    \right).
\]
\end{theorem}

\begin{proof}
By the multiplicative decomposition,
\[
    \sum_{t=1}^d u_t(S)
    =
    -\log \mathrm{Prec}_{\mathrm{seq}}(S)
    \le
    -\log(1-\delta).
\]
Among the $m$ hard steps, fewer than $-\log(1-\delta)/\eta$ steps can have $u_t(S)>\eta$; otherwise the total precision loss would exceed the budget. Therefore at least
\[
    \left(
    m-\frac{-\log(1-\delta)}{\eta}
    \right)_+
\]
hard steps satisfy $u_t(S)\le\eta$. For each such step, hardness implies
$v_t(S)\ge\rho$. Hence
\[
    -\log \mathrm{Rec}_{\mathrm{seq}}(S)
    =
    \sum_{t=1}^d v_t(S)
    \ge
    \rho
    \left(
    m-\frac{-\log(1-\delta)}{\eta}
    \right)_+ .
\]
Exponentiating gives the result.
\end{proof}

\paragraph{Interpretation.}
Maintaining high sequence-level precision imposes a constant total
precision-loss budget across all decoding steps. As the sequence
length grows, most steps must operate in a near-perfect regime.
If many such steps still incur a fixed recall loss, these losses
accumulate multiplicatively, causing an exponential decay in the
set of reachable valid sequences.

\section{Analysis of Theorem~\ref{thm:length}}
\label{app:length_proof}

The main point is simple: if a model must place very high probability on valid
tokens at each step, then its next-token distribution must become sharper. But sharper
distributions are less uniform over the valid choices, which reduces validity-conditioned
diversity. Since entropy losses add over sequence positions, the diversity loss compounds
with length.

\begin{definition}[Discrete geometric ranked model]
\label{def:discrete_geometric_ranked_model}
Fix a temperature \(T>0\). At each valid prefix \(y_{<t}\), assume that after sorting tokens by
decreasing probability, the next-token distribution is geometric in rank:
\[
    P_t^{(T)}(i\mid y_{<t})
    =
    (1-q_t)q_t^i,
    \qquad i=0,1,2,\dots,|\mathcal{V}| - 1
\]
where
\[
    q_t=\exp(-\lambda_t/T),
    \qquad \lambda_t>0.
\]
Equivalently,
\[
    P_t^{(T)}(i\mid y_{<t})
    \propto
    \exp(-\lambda_t i/T).
\]
\end{definition}

For each position \(t\), define the normalized sharpness
\[
    z_t:=\frac{\lambda_t v_t}{T}.
\]
This quantity measures how much the ranked distribution decays across the valid interval
\(\{0,\dots,v_t-1\}\).

For \(v\in\mathbb N_+\) and \(a>0\), define \(H_v(a)\) as the entropy of the tilted
distribution
\[
    p_{v,a}(i)
    =
    \frac{\exp(-ai/v)}
    {\sum_{j=0}^{v-1}\exp(-aj/v)},
    \qquad i=0,\dots,v-1.
\]
When \(a=0\), this distribution is uniform over \(v\) tokens and has entropy \(\ln v\). When
\(a>0\), it is tilted toward smaller ranks, so its entropy is smaller.

Define the per-step entropy loss
\[
    \Delta_v(a)
    :=
    \ln v - H_v(a).
\]
This measures how much diversity is lost, at one step, relative to being uniform over the
\(v\) valid choices.

\begin{definition}[Diversity]
\label{def:validity_conditioned_diversity}
Let \(Y\sim P^{(T)}(\cdot\mid x)\), and let \(V\) be the set of valid sequences. Define
\[
    \mathrm{Div}(P^{(T)})
    :=
    \frac{\exp(H(Y\mid Y\in V))}{|V|}.
\]
This quantity equals \(1\) when the model is uniform over valid sequences after conditioning
on validity. It decreases when the conditional distribution over valid sequences becomes more
concentrated.
\end{definition}

\begin{lemma}[Entropy loss increases with sharpness]
\label{lem:entropy_loss_increases_with_sharpness_clean}
For every fixed \(v\), the entropy loss
\[
    \Delta_v(a)
    =
    \ln v-H_v(a)
\]
is nondecreasing in \(a\). Moreover, if \(v\ge2\) and \(a>0\), then
\[
    \Delta_v(a)>0.
\]
\end{lemma}

\begin{proof}
Write $\theta=a/v$ and
\[
    Z(\theta)=\sum_{j=0}^{v-1}e^{-\theta j}.
\]
For $p_\theta(i)\propto e^{-\theta i}$,
\[
    H_v(a)=\log Z(\theta)+\theta \mathbb E_\theta[i].
\]
Hence
\[
    \frac{d}{d\theta}H_v(a)
    =
    \theta\frac{d}{d\theta}\mathbb E_\theta[i]
    =
    -\theta\operatorname{Var}_\theta(i)
    \le 0.
\]
Since $\theta=a/v$, $H_v(a)$ is nonincreasing in $a$, so
$\Delta_v(a)=\log v-H_v(a)$ is nondecreasing. If $v\ge2$ and $a>0$, then
$p_{v,a}$ is nonuniform, so $H_v(a)<\log v$ and $\Delta_v(a)>0$.
\end{proof}

\begin{theorem}[discrete validity--diversity trade-off]
\label{thm:clean_discrete_validity_diversity_trade-off}
Assume the discrete geometric ranked model and invariant valid
branching. Let
\[
    L:=\ln\epsilon^{-1}.
\]
If
\[
    \mathrm{Val}(P^{(T)})\ge 1-\epsilon,
\]
then
\[
    \mathrm{Div}(P^{(T)})
    \le
    \exp\left(
        -\sum_{t=1}^d \Delta_{v_t}(L)
    \right).
\]
Equivalently,
\[
    \mathrm{Div}(P^{(T)})
    \le
    \exp\left(
        -\sum_{t=1}^d
        \left(\ln v_t-H_{v_t}(\ln\epsilon^{-1})\right)
    \right).
\]

In particular, stricter validity requirements force a smaller upper bound on validity-conditioned diversity: as \(\epsilon\) decreases, \(L=\ln\epsilon^{-1}\) increases,
and each entropy-loss term \(\Delta_{v_t}(L)\) increases.
\end{theorem}

\begin{proof}
At position \(t\), the valid next tokens are the first \(v_t\) ranked tokens. Therefore, the local probability of choosing a valid next token is
\[
\begin{aligned}
    P_t^{(T)}(Y_t\in G_t(y_{<t})\mid y_{<t})
    &=
    \sum_{i=0}^{v_t-1}(1-q_t)q_t^i \\
    &=
    1-q_t^{v_t}.
\end{aligned}
\]
Since
\[
    q_t=\exp(-\lambda_t/T),
\]
we have
\[
    q_t^{v_t}
    =
    \exp(-\lambda_t v_t/T)
    =
    e^{-z_t}.
\]
Thus the local validity probability is
\[
    1-e^{-z_t}.
\]

By invariant valid branching, the full-sequence validity factorizes:
\[
    \mathrm{Val}(P^{(T)})
    =
    \prod_{t=1}^d (1-e^{-z_t}).
\]
Suppose
\[
    \mathrm{Val}(P^{(T)})\ge 1-\epsilon.
\]
Since the product is no larger than any individual factor, for every \(t\),
\[
    1-\epsilon
    \le
    \prod_{s=1}^d(1-e^{-z_s})
    \le
    1-e^{-z_t}.
\]
Therefore
\[
    e^{-z_t}\le \epsilon,
\]
which implies
\[
    z_t\ge \ln\epsilon^{-1}=L.
\]

Now condition on the event that the generated sequence is valid. At position \(t\), conditioned on choosing one of the \(v_t\) valid tokens, the rank distribution is
\[
    p_{v_t,z_t}(i)
    =
    \frac{\exp(-z_t i/v_t)}
    {\sum_{j=0}^{v_t-1}\exp(-z_t j/v_t)},
    \qquad i=0,\dots,v_t-1.
\]
Its entropy is \(H_{v_t}(z_t)\).

By the chain rule for entropy,
\[
    H(Y\mid Y\in V)
    =
    \sum_{t=1}^d H_{v_t}(z_t).
\]
Also, by invariant valid branching,
\[
    |V|=\prod_{t=1}^d v_t.
\]
Therefore
\[
\begin{aligned}
    \mathrm{Div}(P^{(T)})
    &=
    \frac{\exp(H(Y\mid Y\in V))}{|V|} \\
    &=
    \exp\left(
        \sum_{t=1}^d H_{v_t}(z_t)
        -
        \sum_{t=1}^d \ln v_t
    \right) \\
    &=
    \exp\left(
        -
        \sum_{t=1}^d
        \bigl(\ln v_t-H_{v_t}(z_t)\bigr)
    \right).
\end{aligned}
\]

By Lemma~\ref{lem:entropy_loss_increases_with_sharpness_clean}, entropy loss is nondecreasing in sharpness. Since \(z_t\ge L\), we have
\[
    \ln v_t-H_{v_t}(z_t)
    \ge
    \ln v_t-H_{v_t}(L)
    =
    \Delta_{v_t}(L).
\]
Plugging this into the previous display gives
\[
    \mathrm{Div}(P^{(T)})
    \le
    \exp\left(
        -\sum_{t=1}^d \Delta_{v_t}(L)
    \right).
\]
This proves the desired bound.

Finally, because each \(\Delta_{v_t}(L)\) is nondecreasing in \(L\), and because
\(L=\ln\epsilon^{-1}\) increases as \(\epsilon\) decreases, stricter validity requirements force a smaller upper bound on validity-conditioned diversity.
\end{proof}

\begin{corollary}[Exponential diversity loss in branching length]
\label{cor:exponential_diversity_loss_branching_length}
Let
\[
    m := \left|\{t\in[d]: v_t\ge2\}\right|
\]
be the branching length, and define
\[
    c_{\mathcal V}(\epsilon)
    :=
    \min_{2\le v\le |\mathcal V|}
    \Delta_v(\ln\epsilon^{-1}).
\]
For every $\epsilon\in(0,1)$, we have $c_{\mathcal V}(\epsilon)>0$ and
\[
    \mathrm{Div}(P^{(T)})
    \le
    \exp\!\left(-m\,c_{\mathcal V}(\epsilon)\right).
\]
In particular, if every position has at least two valid next-token choices, then $m=d$, and validity-conditioned diversity decays exponentially in sequence length.
\end{corollary}

\begin{proof}
Theorem~\ref{thm:clean_discrete_validity_diversity_trade-off} gives
\[
    \mathrm{Div}(P^{(T)})
    \le
    \exp\!\left(
        -\sum_{t=1}^d \Delta_{v_t}(\ln\epsilon^{-1})
    \right).
\]
Terms with $v_t=1$ contribute no branching diversity. For every term with
$v_t\ge2$,
\[
    \Delta_{v_t}(\ln\epsilon^{-1})
    \ge
    c_{\mathcal V}(\epsilon).
\]
There are $m$ such terms, so
\[
    \sum_{t=1}^d \Delta_{v_t}(\ln\epsilon^{-1})
    \ge
    m\,c_{\mathcal V}(\epsilon).
\]
Finally, $c_{\mathcal V}(\epsilon)>0$ because it is the minimum of finitely many strictly positive entropy losses. This proves the claim.
\end{proof}

\begin{corollary}[Two regimes of the diversity loss]
\label{cor:two_regimes_diversity_loss}
Let \(L:=\ln\epsilon^{-1}\), and define
\[
    c_{\mathcal V}(\epsilon)
    :=
    \min_{2\le v\le |\mathcal V|}
    \Delta_v(L).
\]
Then
\[
    \mathrm{Div}(P^{(T)})
    \le
    \exp\!\left(-m\,c_{\mathcal V}(\epsilon)\right).
\]
Moreover, \(c_{\mathcal V}(\epsilon)\) has the following two regimes:
\[
    c_{\mathcal V}(\epsilon)
    =
    \frac{1}{32}L^2+O(L^3),
    \qquad
    L\to0,
\]
and
\[
    c_{\mathcal V}(\epsilon)
    \to
    \ln 2,
    \qquad
    L\to\infty.
\]
Equivalently, in the weak-validity regime,
\[
    \mathrm{Div}(P^{(T)})
    \le
    \exp\!\left(
        -m\left(\frac{1}{32}L^2+O(L^3)\right)
    \right),
\]
while in the stringent-validity regime,
\[
    \mathrm{Div}(P^{(T)})
    \le
    \exp\!\left(-m(\ln2-o(1))\right)
    =
    2^{-m+o(m)} .
\]
\end{corollary}

\begin{proof}
The first inequality follows directly from Corollary~\ref{cor:exponential_diversity_loss_branching_length} with
\(c(\epsilon)\) replaced by the finite-vocabulary minimum
\(c_{\mathcal V}(\epsilon)\).

We now prove the two asymptotic regimes. Recall that
\[
    \Delta_v(L)
    =
    D_{\mathrm{KL}}\!\left(p_{v,L}\,\middle\|\,U_v\right),
\]
where \(U_v\) is uniform on \(\{0,\dots,v-1\}\) and
\(p_{v,L}(i)\propto e^{-Li/v}\).

For \(L\to0\), \(p_{v,L}\) is a small exponential tilt of \(U_v\). The standard second-order expansion of KL divergence gives
\[
    \Delta_v(L)
    =
    \frac{L^2}{2}
    \operatorname{Var}_{U_v}\!\left(\frac{i}{v}\right)
    +
    O(L^3).
\]
Since
\[
    \operatorname{Var}_{U_v}\!\left(\frac{i}{v}\right)
    =
    \frac{v^2-1}{12v^2},
\]
we obtain
\[
    \Delta_v(L)
    =
    \frac{v^2-1}{24v^2}L^2
    +
    O(L^3).
\]
The coefficient
\[
    \frac{v^2-1}{24v^2}
\]
is increasing in \(v\ge2\), so the minimum over
\(2\le v\le |\mathcal V|\) is attained at \(v=2\). Hence
\[
    c_{\mathcal V}(\epsilon)
    =
    \Delta_2(L)
    =
    \frac{1}{32}L^2+O(L^3).
\]

For \(L\to\infty\), fix any finite \(v\). Then \(p_{v,L}\) concentrates on the top-ranked token, so
\[
    H_v(L)\to0,
    \qquad
    \Delta_v(L)=\ln v-H_v(L)\to \ln v.
\]
Because the minimum is over the finite set \(\{2,\dots,|\mathcal V|\}\), we may pass the limit through the minimum:
\[
    c_{\mathcal V}(\epsilon)
    =
    \min_{2\le v\le |\mathcal V|}\Delta_v(L)
    \to
    \min_{2\le v\le |\mathcal V|}\ln v
    =
    \ln2.
\]
Substituting these two asymptotics into
\[
    \mathrm{Div}(P^{(T)})
    \le
    \exp(-m\,c_{\mathcal V}(\epsilon))
\]
gives the claimed diversity bounds.
\end{proof}

\paragraph{Interpretation.}
The result shows that high validity requires every local invalid-token probability to be small. In the geometric ranked model, this forces large normalized sharpness $z_t$. After conditioning on validity, the distribution over the $v_t$ valid choices is therefore
tilted rather than uniform, causing an entropy loss $\Delta_{v_t}(z_t)$ at each branching position. Since entropy losses add over positions, the exponentiated diversity decays
multiplicatively, yielding
\[
    \mathrm{Div}(P^{(T)})
    \le
    \exp(-m\,c_{\mathcal V}(\epsilon)).
\]

\section{Experiments with Production-Level Models}\label{app:gpt}

In Figure~\ref{fig:gpt_city}, we prompted \texttt{GPT-5.5} with the prompt ``Name a random city in the world'', with default thinking level and temperature. The models response almost always lie in a limited set of few cities.

\begin{figure}[t]
    \centering
    \includegraphics[width=\linewidth]{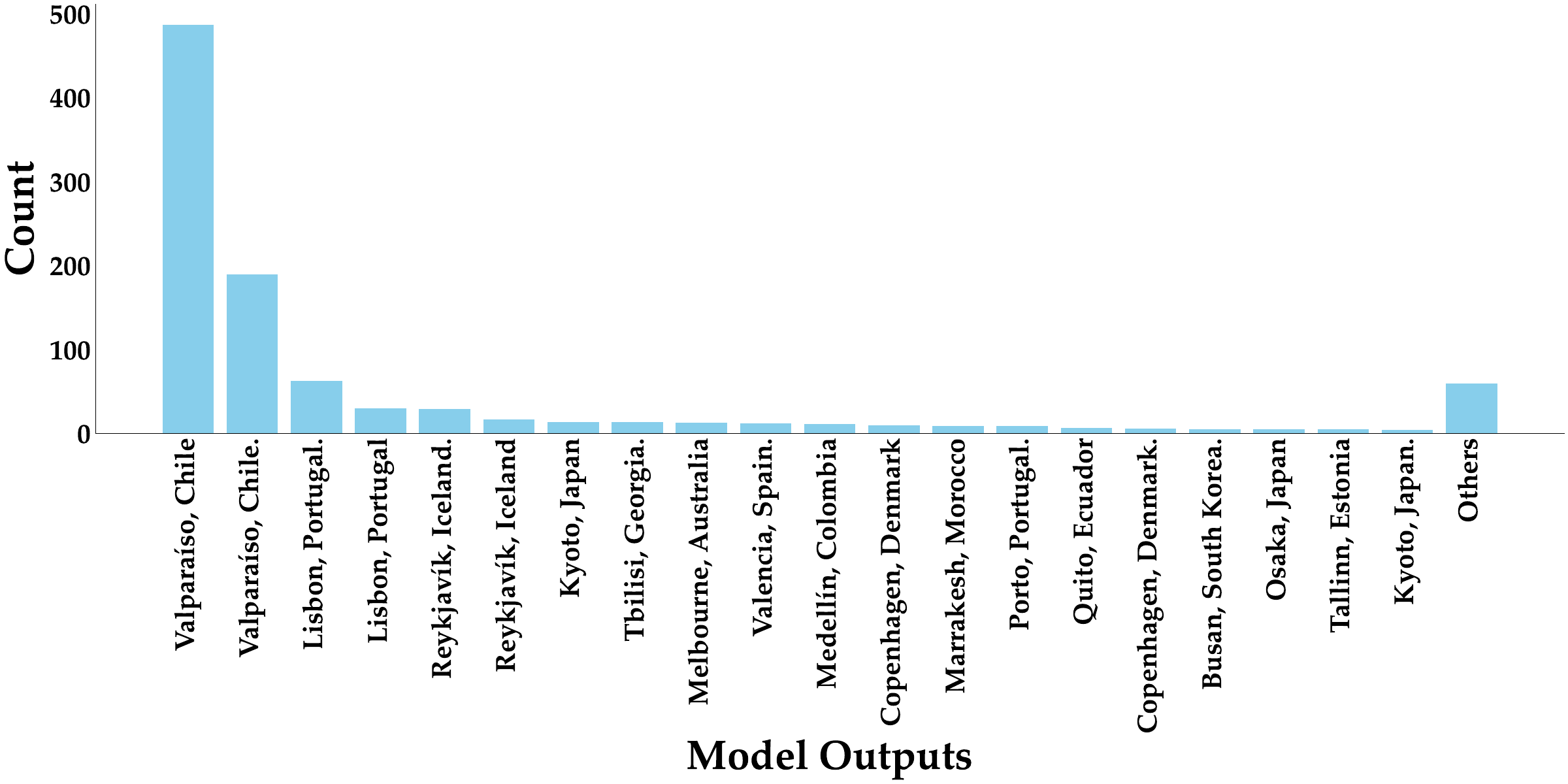}
    \caption{Prompting \texttt{GPT-5.5} to randomly name a city in the world. The vast majority of answers, with or without user chat history, collapses to ``Valparaíso, Chile.'' This shows a strong collapse in diversity.}
    \label{fig:gpt_city}
\end{figure}

\section{Case Study: Calibration in Coding}
\label{app:coding}
We examine the distributional properties of LLMs in a simple coding setting. Although diversity is not inherently the primary objective in coding tasks, prior work has shown that greedy decoding is often suboptimal. Instead, \citet{yue2025limit-of-rlvr} advocate evaluating models using the \textit{pass@k} metric, which measures whether at least one of $k$ sampled solutions correctly solves the task. This highlights the importance of effective sampling even in domains where a single correct answer suffices.

Figure~\ref{fig} illustrates sequence-level probability distributions on a coding task from LiveCodeBench~\citep{jain2025livecodebench}. We observe clear evidence of both order and shape miscalibration. Valid solutions are interspersed with invalid ones throughout the ranked distribution~\citep{karan2025reasoningsamplingbasemodel}, indicating poor order calibration. At the same time, probability mass is unevenly concentrated across valid solutions, with a small subset dominating the distribution, reflecting shape miscalibration. 

These results show that our framework helps diagnose model failures beyond diversity alone, providing a lens to understand why sampling-based improvements such as higher \textit{pass@k} remain difficult to achieve in practice.

\begin{figure}[t]
\centering
\includegraphics[width=\linewidth]{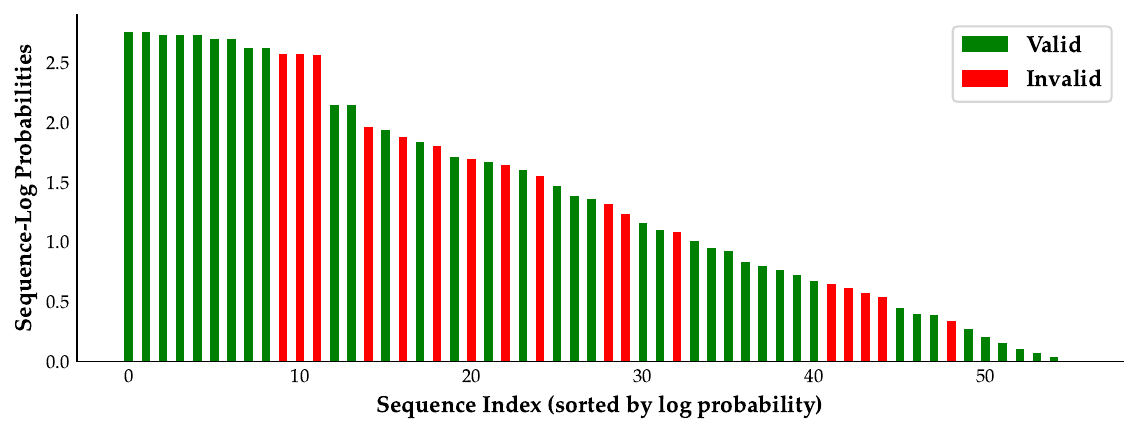}
\caption{Sequence-level probability distribution for a coding task from LiveCodeBench~\citep{jain2025livecodebench}, with sequences sorted by log-probability. Valid and invalid solutions are intermixed across the ranking (order miscalibration), and probability mass is concentrated on a small subset of valid solutions (shape miscalibration).}
\label{fig}
\end{figure}

\section{Empirical Analysis of Logits}
\label{app:logits_fit}
To investigate the logits distribution, we fit the logits at each generation conditional to a piecewise model, defined as 
\[
f(k) =
\begin{cases}
mk + b, & k \le c, \\
A + B \log(k + C), & k > c,
\end{cases}
\]
where $k$ refers to the logit index and $f(k)$ the corresponding logit value. To find $c$, we sweep over all token indexes. Figures~\ref{fig:logits_llama}, \ref{fig:logits_qwen}, \ref{fig:logits_olmo} show the curve fit and MSE, R$^2$ of fitting to various tasks and conditionals. Results show that LLM conditionals are consistently sharp-headed and heavy-tailed.

\begin{figure}[!h]
    \centering
      \begin{subfigure}[!t]{0.99\textwidth}
        \includegraphics[width=1.0\linewidth]{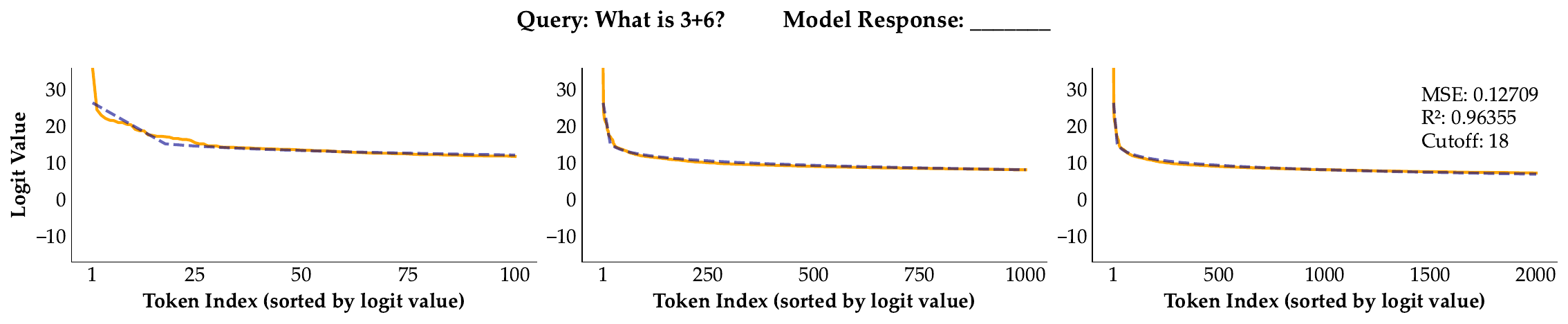}
      \end{subfigure}
       \hfill
       \vspace{5mm}
      \begin{subfigure}[!t]{0.99\textwidth}
        \includegraphics[width=1.0\linewidth]{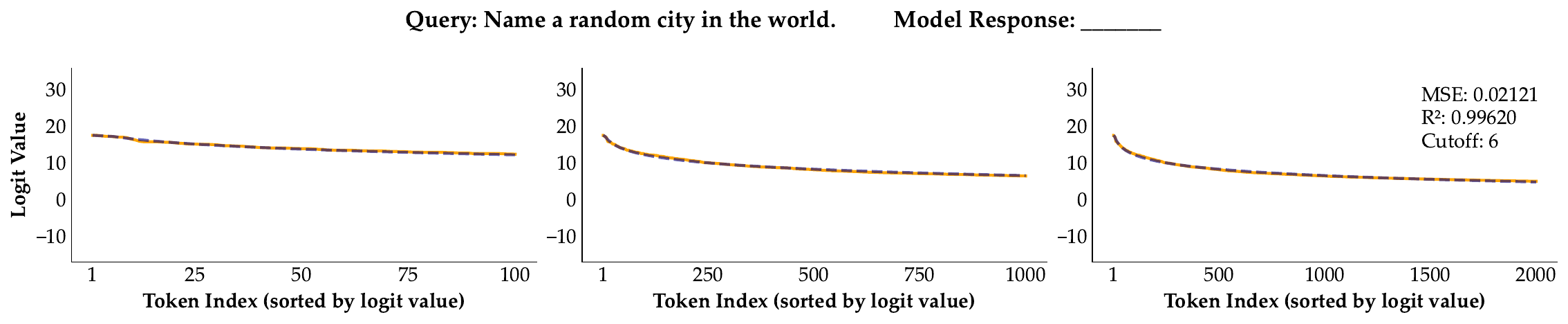}
      \end{subfigure}
       \hfill
       \vspace{5mm}
      \begin{subfigure}[!t]{0.99\textwidth}
        \includegraphics[width=1.0\linewidth]{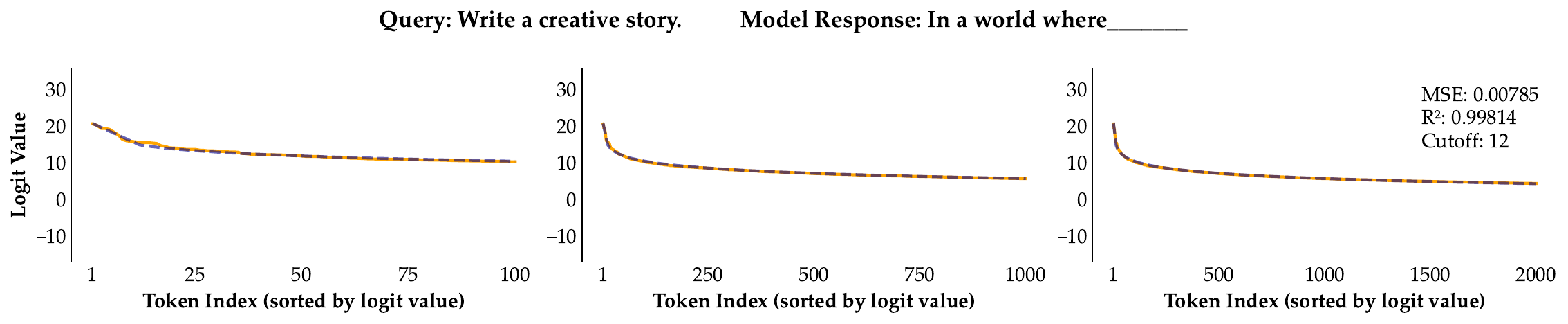}
      \end{subfigure}
       \hfill
       \vspace{5mm}
      \begin{subfigure}[!t]{0.99\textwidth}
        \includegraphics[width=1.0\linewidth]{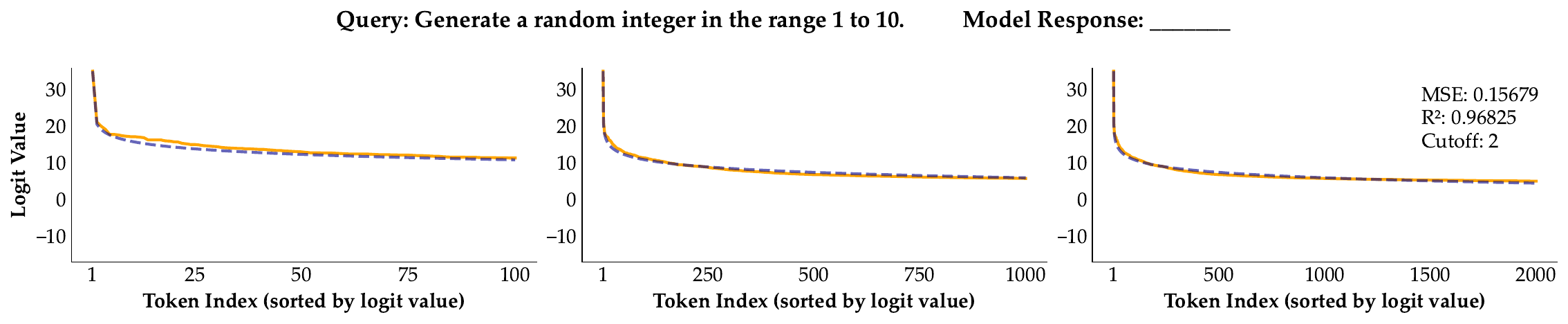}
      \end{subfigure}
       \hfill
       \vspace{5mm}
      \begin{subfigure}[!t]{0.99\textwidth}
        \includegraphics[width=1.0\linewidth]{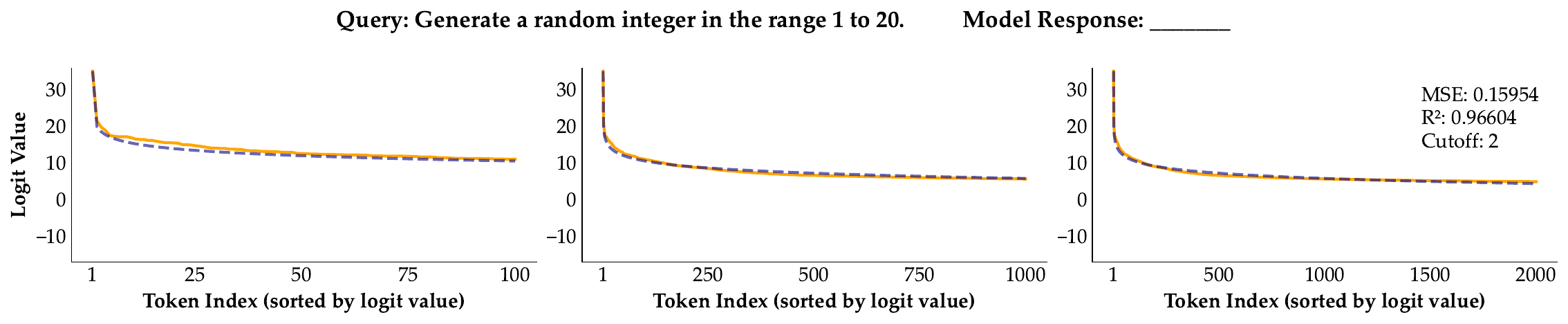}
      \end{subfigure}
       \hfill
       \vspace{5mm}
      \begin{subfigure}[!t]{0.99\textwidth}
        \includegraphics[width=1.0\linewidth]{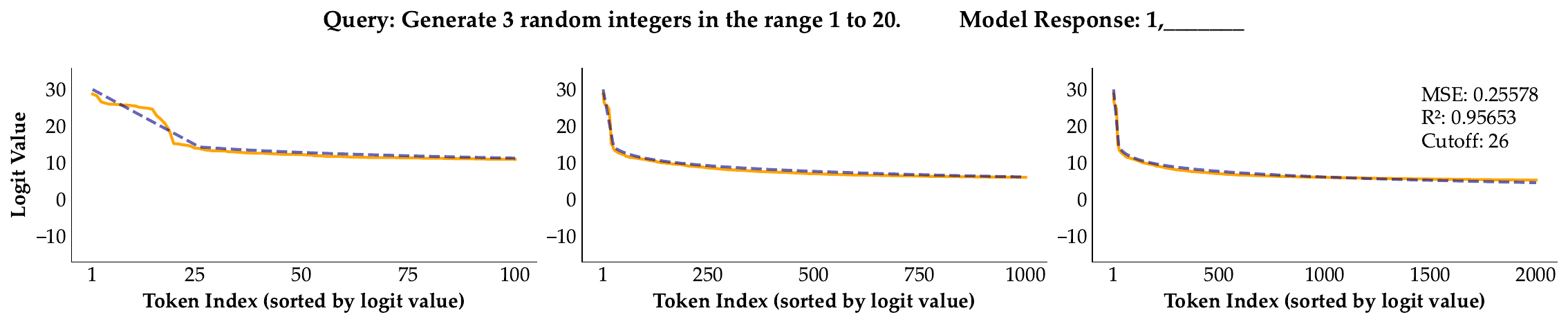}
      \end{subfigure}
       \hfill       
    \caption{Logit fitting on \texttt{Llama-3.1-8B-Instruct}}
    \label{fig:logits_llama}
\end{figure}

\begin{figure}[!h]
    \centering
      \begin{subfigure}[!t]{0.99\textwidth}
        \includegraphics[width=1.0\linewidth]{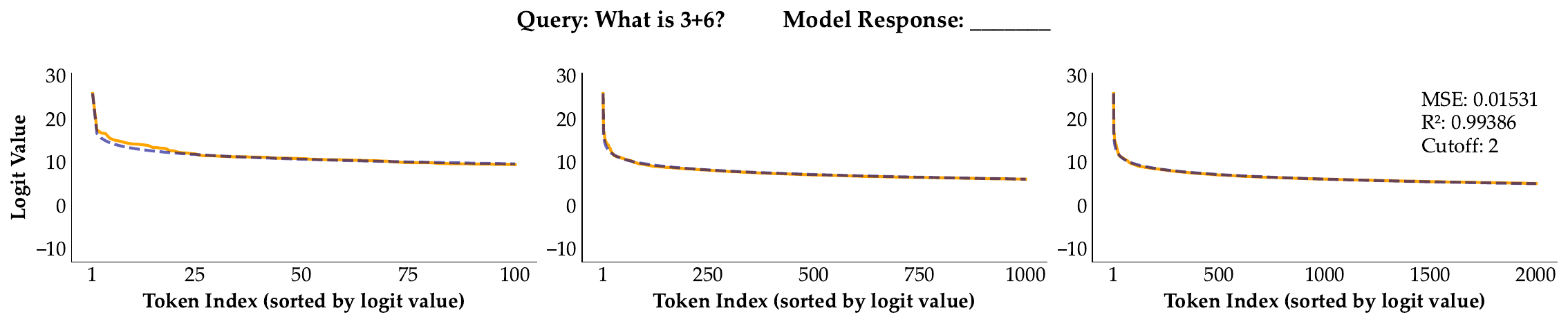}
      \end{subfigure}
       \hfill
       \vspace{5mm}
      \begin{subfigure}[!t]{0.99\textwidth}
        \includegraphics[width=1.0\linewidth]{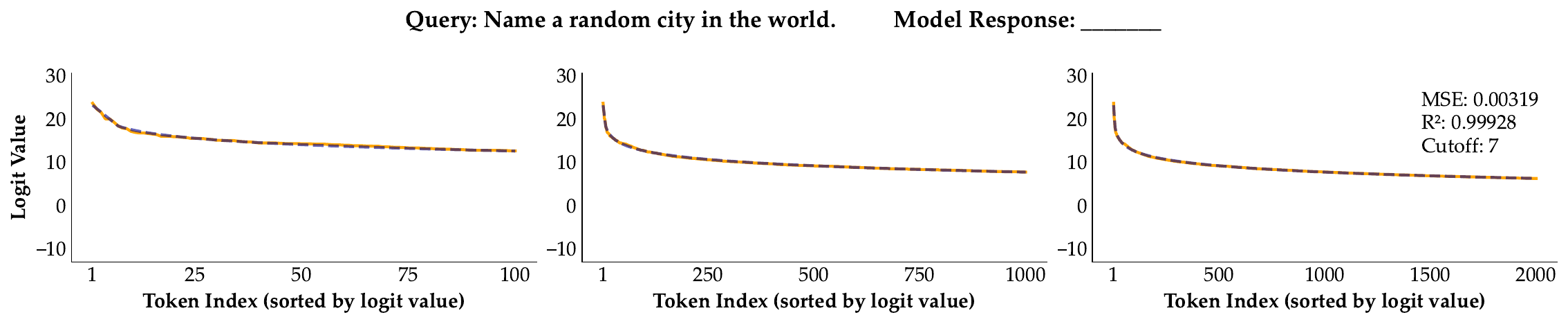}
      \end{subfigure}
       \hfill
       \vspace{5mm}
      \begin{subfigure}[!t]{0.99\textwidth}
        \includegraphics[width=1.0\linewidth]{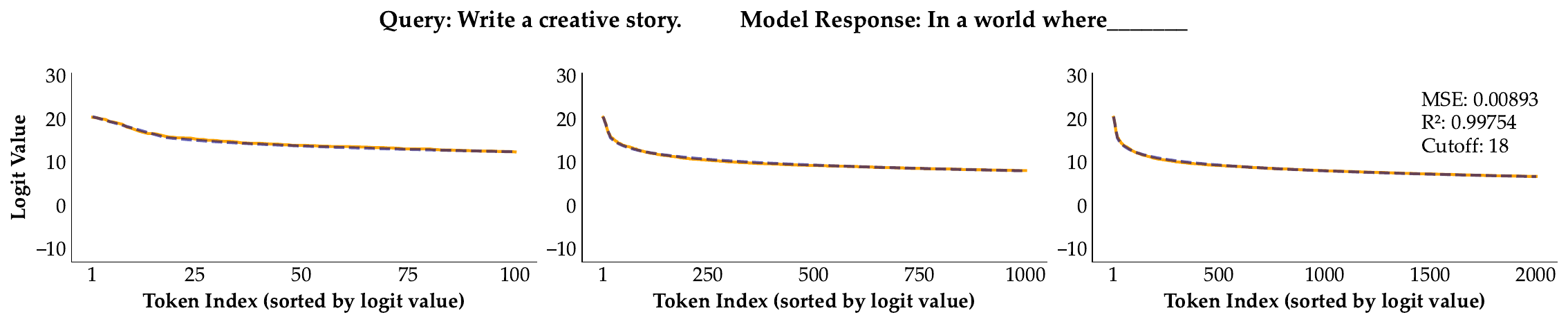}
      \end{subfigure}
       \hfill
       \vspace{5mm}
      \begin{subfigure}[!t]{0.99\textwidth}
        \includegraphics[width=1.0\linewidth]{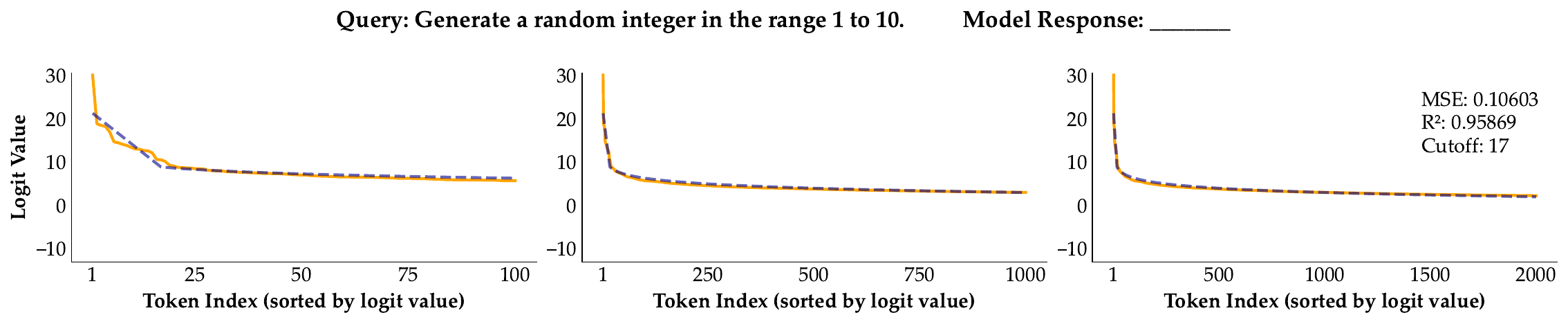}
      \end{subfigure}
       \hfill
       \vspace{5mm}
      \begin{subfigure}[!t]{0.99\textwidth}
        \includegraphics[width=1.0\linewidth]{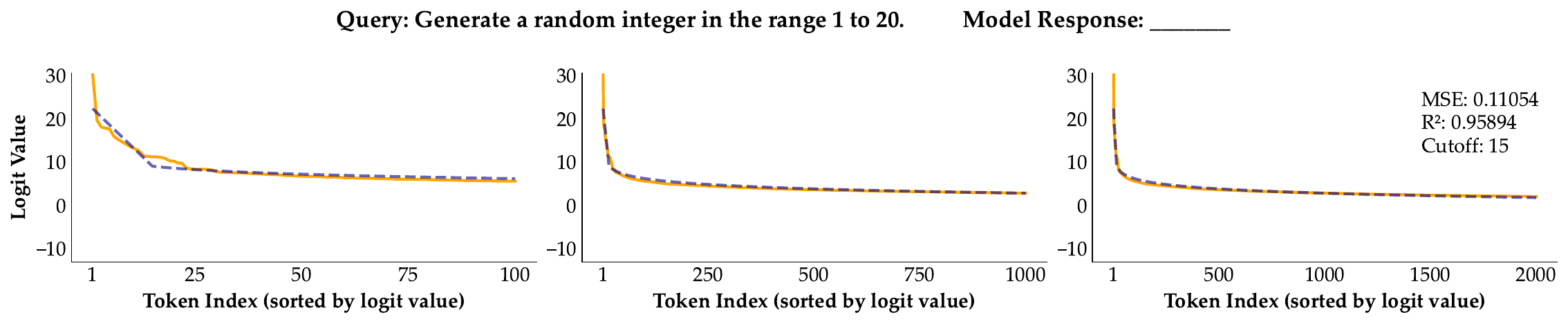}
      \end{subfigure}
       \hfill
       \vspace{5mm}
      \begin{subfigure}[!t]{0.99\textwidth}
        \includegraphics[width=1.0\linewidth]{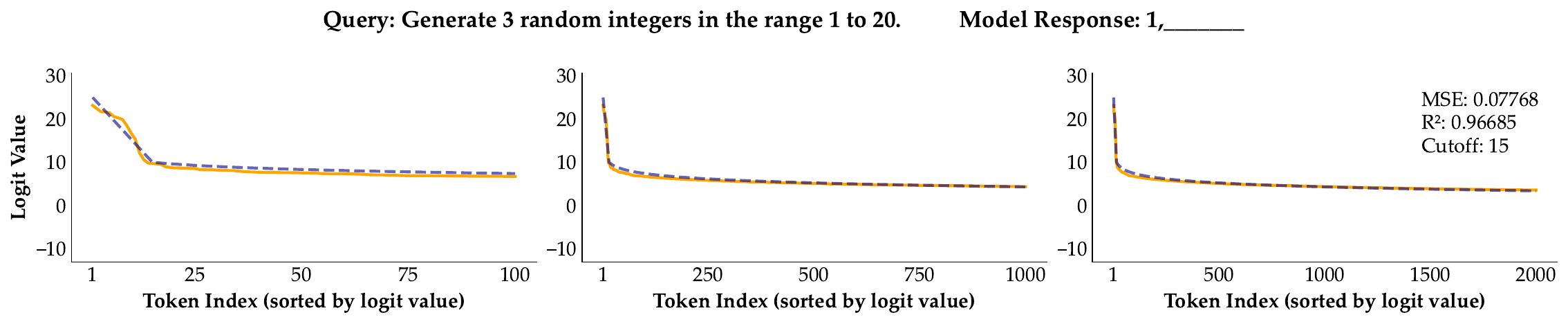}
      \end{subfigure}
       \hfill       
    \caption{Logit fitting on \texttt{Qwen3.5-35B-A3B}}
    \label{fig:logits_qwen}
\end{figure}

\begin{figure}[!h]
    \centering
      \begin{subfigure}[!t]{0.99\textwidth}
        \includegraphics[width=1.0\linewidth]{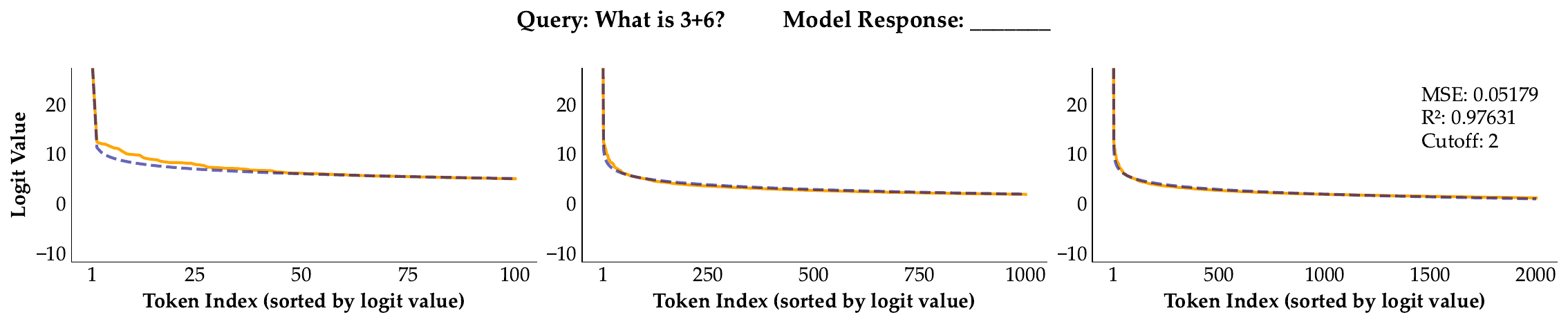}
      \end{subfigure}
       \hfill
       \vspace{5mm}
      \begin{subfigure}[!t]{0.99\textwidth}
        \includegraphics[width=1.0\linewidth]{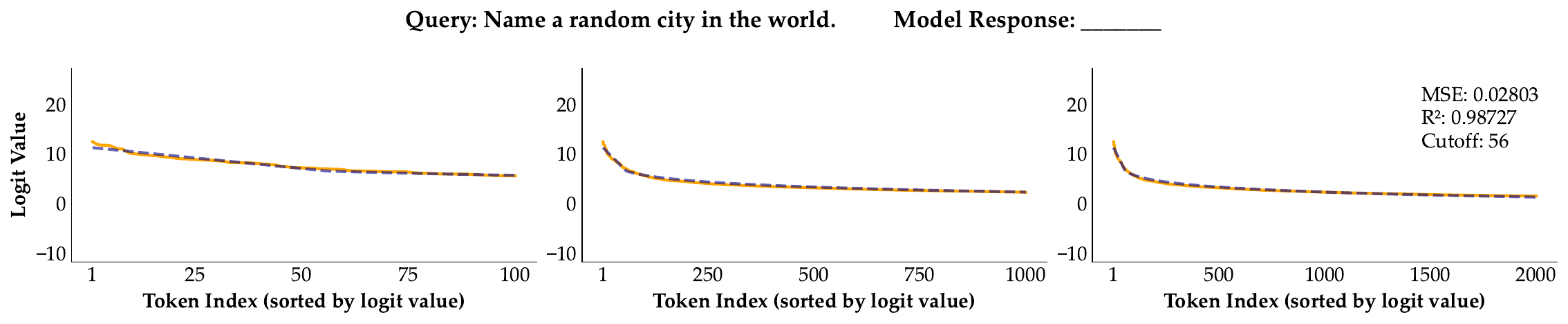}
      \end{subfigure}
       \hfill
       \vspace{5mm}
      \begin{subfigure}[!t]{0.99\textwidth}
        \includegraphics[width=1.0\linewidth]{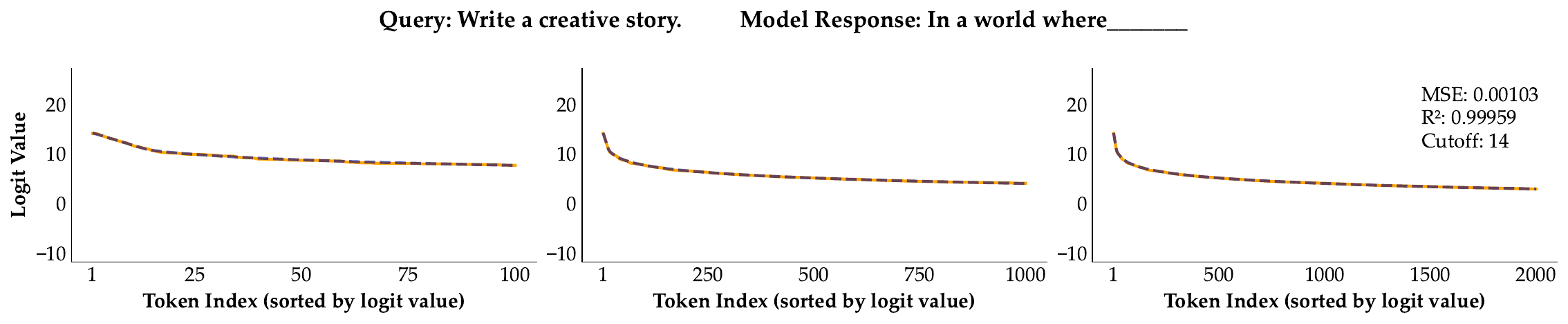}
      \end{subfigure}
       \hfill
       \vspace{5mm}
      \begin{subfigure}[!t]{0.99\textwidth}
        \includegraphics[width=1.0\linewidth]{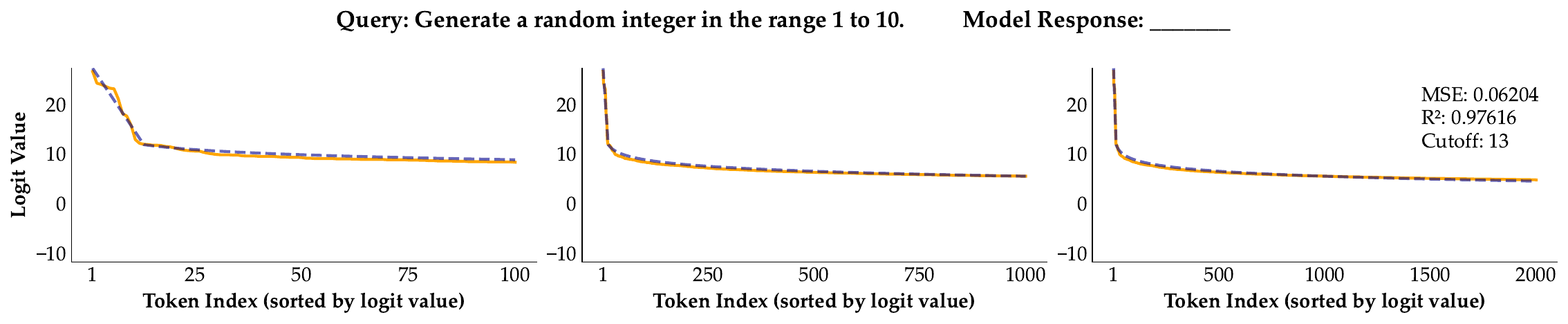}
      \end{subfigure}
       \hfill
       \vspace{5mm}
      \begin{subfigure}[!t]{0.99\textwidth}
        \includegraphics[width=1.0\linewidth]{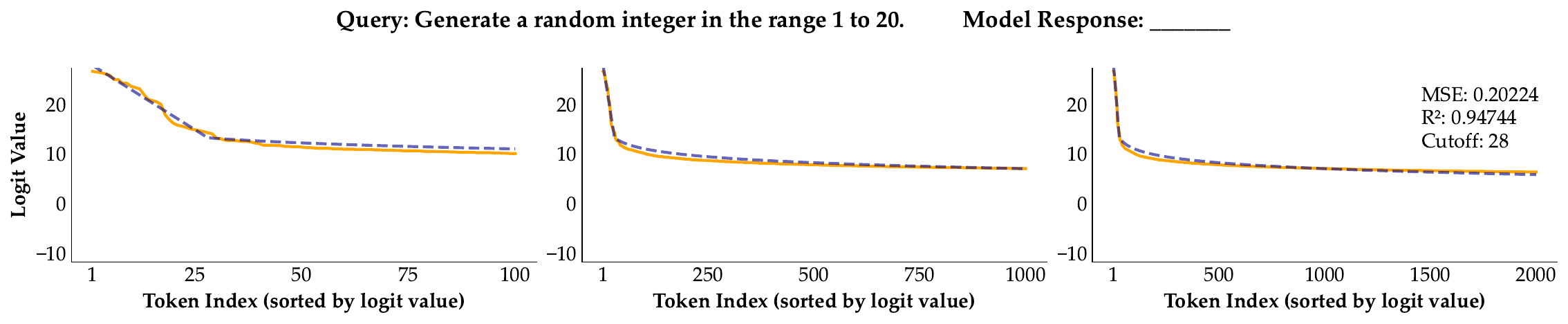}
      \end{subfigure}
       \hfill
       \vspace{5mm}
      \begin{subfigure}[!t]{0.99\textwidth}
        \includegraphics[width=1.0\linewidth]{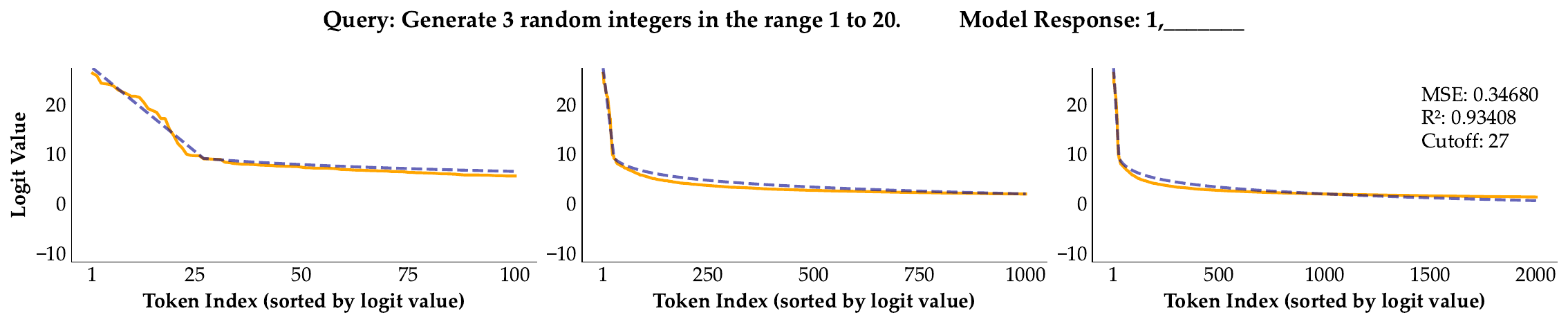}
      \end{subfigure}
       \hfill       
    \caption{Logit fitting on \texttt{Olmo-3-7B-Instruct}}
    \label{fig:logits_olmo}
\end{figure}

\section{Limitations}
\label{app:limitations}

Our work studies diversity collapse through the lens of validity--diversity
calibration. This perspective is useful because it connects output-level diversity failures to local properties of the model's conditional distributions. However, several limitations should be noted.

\paragraph{Controlled tasks are diagnostic rather than exhaustive.}
A substantial part of our empirical analysis uses controlled random-number generation tasks, where the valid set can be characterized exactly. This allows us to compute validity and diversity without relying on noisy semantic judgments, and makes it possible to isolate the effects of shape and order calibration. However, these tasks are not intended to capture the full complexity of open-ended generation. In domains such as creative writing, scientific ideation, dialogue, or planning, validity is semantic, context-dependent, and often graded rather than binary. Therefore, our controlled experiments should be interpreted as diagnostics for specific distributional mechanisms, not as a complete account of all diversity failures in realistic generation settings.

\paragraph{Token-level validity labels are approximate in open-ended settings.}
For open-ended tasks, the valid-token set is not directly observable. Our empirical procedure approximates token validity by extending a prefix with a candidate token, greedily decoding a completion, and then evaluating the final sequence with an LLM judge. This provides a practical estimate of whether a token preserves access to a valid continuation, but it is not an exact characterization of the true valid-token set: a token that leads to an invalid greedy continuation may still admit valid continuations under another decoding path.

To assess the reliability of this approximation, we conduct a human validation study comparing judge-based labels against human annotations. This helps quantify the extent to which the LLM judge agrees with human judgments and reduces concern that our precision--recall estimates are artifacts of a particular judge model. Nevertheless, the labels remain approximate, since both human and model judgments depend on the task rubric, validity threshold, and the particular continuation used for evaluation. Our controlled random-number experiments avoid this source of noise by using algorithmically known validity.

\paragraph{The theoretical models isolate mechanisms under stylized assumptions.}
Our theoretical results are designed to formalize clean mechanisms rather than to fully model all LLM distributions. In particular, the shape calibration analysis uses a ranked geometric distribution and, in its cleanest form, assumes invariant valid branching across prefixes. These assumptions make it possible to show how local sharpness and entropy losses compound across sequence positions. Real LLM conditionals may have heterogeneous branching factors, non-geometric tails, prefix-dependent valid sets, and interactions between syntax, semantics, and instruction-following constraints. The theory should therefore be read as a mechanistic explanation of why validity--diversity trade-offs can arise, rather than as a literal generative model of all LLM behavior.

\paragraph{Oracle baselines are diagnostic, not deployable.}
Several experiments use oracle information, such as the ground-truth valid-token set size or exact validity constraints in controlled tasks. These oracle baselines are not meant to be practical decoding methods. Their purpose is to separate failure modes. For example, an oracle-size cutoff tests whether a rank-based method would improve if it knew the correct local support size, while still failing when valid and invalid tokens are interleaved in rank. Thus, oracle performance should be interpreted as evidence about the source of the bottleneck, not as a directly available inference-time algorithm.

\paragraph{Sequence-level experiments are limited by sequence depth.}
Exact validity--diversity computation becomes expensive as sequence length grows, because the number of possible continuations increases rapidly. This limits the lengths and branching structures that can be exhaustively evaluated in controlled settings. Our experiments therefore emphasize short-to-moderate horizons where exact computation is feasible. The theory predicts that the relevant losses compound with the number of branching positions, but larger-scale empirical validation over longer sequences remains an important direction for future work.

\paragraph{This work diagnoses rather than solves the bottleneck.}
Our goal is to identify distributional mechanisms that constrain validity and diversity during decoding. We do not propose a new decoding algorithm or training objective that fully resolves these issues. The results suggest that effective solutions may need to improve the model's conditional calibration directly, or use auxiliary validity signals that go beyond probability rank. Designing such calibration-aware training or decoding methods is left for future work.

\section{Compute Resources}
\label{app:resources}
We conduct our experiments on 8 NVIDIA A6000 GPUs. For experiments on commercial models, we use the corresponding official API endpoints. Open-source LLM inference is conducted through the vLLM \citep{kwon2023efficient} package. The total computing time required to reproduce all our generations and evaluate the results is around 1 week. 

\section{Societal Impacts}
\label{app:societal_impacts}
Our work directly addresses the frontiers of LLM inference. In particular, our work aims to understand the validity and diversity of LLM-generated outputs through fine-grained token control. Although we observe no evidence for harmful content in our experiments, the token-level control could potentially lead to instances where sensitive/harmful content is produced. However, our method does not attempt to jailbreak or induce unsafe behavior in any model. 

On a high level, work incentivizing LLM generation diversity can produce novel ideas and solutions to societal problems, but also increase risk of biased content. Advancing model diversity and capability while guaranteeing trustworthiness and safety remains a high priority in this and all future works. 

\label{app:impacts}
%%%%%%%%%%%%%%%%%%%%%%%%%%%%%%%%%%%%%%%%%%%%%%%%%%%%%%%%%%%%

% \newpage
% \input{checklist.tex}

%%%%%%%%%%%%%%%%%%%%%%%%%%%%%%%%%%%%%%%%%%%%%%%%%%%%%%%%%%%%

\end{document}